%% file: QoE_main.tex
\documentclass[10pt,journal]{IEEEtran}
\ifCLASSOPTIONcompsoc
  \usepackage[nocompress]{cite}
\else
  \usepackage{cite}
\fi
\ifCLASSINFOpdf
\else
\fi
\usepackage{graphicx}
\usepackage{diagbox}
\usepackage{amsfonts}
\usepackage{bbm}
\usepackage{amssymb}
\usepackage{mathrsfs}
\usepackage{mathtools}
\usepackage{graphicx}
\usepackage{subfigure}
\usepackage{enumerate}
\usepackage{amsmath}
\usepackage{color}
\usepackage{amsthm}
\usepackage{algorithm}
\usepackage{algpseudocode}
\usepackage{accents}
\usepackage{stfloats}
\usepackage{subfigure}
\usepackage{caption}
\usepackage{wrapfig}
\usepackage{lipsum}
\input{macros}

\begin{document}
%
\title{Automated Customization of On-Thing Inference for Quality-of-Experience Enhancement}
%
%
%
%

\author{Yang~Bai,~\IEEEmembership{Student Member,~IEEE,}
        Lixing~Chen,~\IEEEmembership{Member,~IEEE,}
        Shaolei~Ren,~\IEEEmembership{Senior Member,~IEEE,}
        and~Jie~Xu,~\IEEEmembership{Senior Member,~IEEE}

\IEEEcompsocitemizethanks{
Y. Bai and J. Xu are with the Department of Electrical and Computer Engineer, University of Miami, FL, 33146, USA. E-mail:y.bai9@umiami.edu, jiexu@miami.edu.

L. Chen is with the Institute of Cyber Science and Technology, Shanghai Jiao Tong University, and Shanghai Key Laboratory of Integrated Administration Technologies for Information Security, Shanghai 200240. E-mail: lxchen@sjtu.edu.cn.
 
S. Ren is with the Electrical and Computer Engineering, University of California, Riverside, CA, 92521, USA. E-mail:sren@ece.ucr.edu.}
}

%
%

\markboth{Journal of \LaTeX\ Class Files,~Vol.~0, No.~0, Decemner~2021}%
{Shell \MakeLowercase{\textit{et al.}}: Bare Demo of IEEEtran.cls for Computer Society Journals}
%

\maketitle



\IEEEtitleabstractindextext{%
\begin{abstract}
	The rapid uptake of intelligent applications is pushing deep learning (DL) capabilities to Internet-of-Things (IoT). Despite the emergence of new tools for embedding deep neural networks (DNNs) into IoT devices, providing satisfactory Quality of Experience (QoE) to users is still challenging due to the heterogeneity in DNN architectures, IoT devices, and user preferences. This paper studies automated customization for DL inference on IoT devices (termed as on-thing inference), and our goal is to enhance user QoE by configuring the on-thing inference with an appropriate DNN for users under different usage scenarios. The core of our method is a DNN selection module that learns user QoE patterns on-the-fly and identifies the best-fit DNN for on-thing inference with the learned knowledge. It leverages a novel online learning algorithm, \emph{NeuralUCB}, that has excellent generalization ability for handling various user QoE patterns. We also embed the knowledge transfer technique in NeuralUCB to expedite the learning process. However, NeuralUCB frequently solicits QoE ratings from users, which incurs non-negligible inconvenience. To address this problem, we design feedback solicitation schemes to reduce the number of QoE solicitations while maintaining the learning efficiency of NeuralUCB. A pragmatic problem, \emph{aggregated QoE}, is further investigated to improve the practicality of our framework. We conduct experiments on both synthetic and real-world data. The results indicate that our method efficiently learns the user QoE pattern with few solicitations and provides drastic QoE enhancement for IoT devices.  
\end{abstract}

}

\IEEEdisplaynontitleabstractindextext

%
\IEEEpeerreviewmaketitle

\section{Introduction}
Deep learning (DL) revolutionizes a broad spectrum of domains \cite{zhang2019deep, deng2014deep} --- computer vision, natural language processing, healthcare, autonomous driving to name a few, achieving remarkable performance comparable to or even exceeding human levels. Such DL intelligence is however built on a deep pool of computing resources. A typical deep neural network (DNN) can contain millions of or even more parameters and consequently, the training of DNNs is often carried out in cloud-scale data centers \cite{eshratifar2019jointdnn, li2017multi}. There is no doubt that Cloud has been a blessing platform for generating DL intelligence, but as the Internet-of-Things (IoT) and mobile industry prospers, there is a growing trend to push the DL intelligence toward end-users as close as possible \cite{ran2018deepdecision, freire2019deep,lane2015can}. More precisely, the inference of DNN is brought down to end devices and IoT devices, such as smartphones, wearables, drones, and medical devices. For example, Apple Siri uses DL for speech synthesis on smartphones and smart watches, and Facebook keeps improving the experience of DNN inference in its mobile app for more than 2 billion users \cite{wu2019machine}. Running DNN inferences on IoT devices (hereinafter, referred to as on-thing inference), users receive improved service quality with reduced service latency, and the inference process also becomes less dependent on Cloud or other edge computing platforms \cite{mao2017mobile}. 

The huge market of IoT and mobile industries is continuously driving the advance of DL techniques for IoT devices. New-generation hardware, e.g., Apple neural engine \cite{sima2018apple}, is designed to accelerate DNN processing on chips. Lightweight DL libraries (e.g., Tensorflow Lite \cite{lite2017android} and Core ML \cite{coreML}) are built to support DL applications on IoT devices. Novel DL algorithms (e.g. 
DNN compression and knowledge distillation \cite{xie2019source, zhang2018systematic, wang2019private}) are proposed to compress large-scale DNN models into compact models that are computationally feasible for embedded devices. As on-thing inference is becoming feasible, a natural to-do item is to improve its inference performance. Various mechanisms, e.g., model selection \cite{taylor2018adaptive, park2015big} DNN compression \cite{xie2019source,zhang2018systematic}, and hyperparameter optimization \cite{cai2017neuralpower,rouhani2016delight}) have been studied in the context of on-device/on-thing inference to deliver higher Quality-of-Service (QoS) in terms of accuracy, inference delay, and energy consumption. While QoS is a useful metric for measuring technical performances, it only reflects part of the service quality. Both industry and academia are moving away from these simple metrics and embracing more holistic approaches of monitoring the overall customer experience \cite{chen2014qos}. Unlike QoS that focuses on a specific set of measurable performance metrics, Quality of Experience (QoE) looks at the overall user satisfaction about the provided service, a fuzzier domain where certain performance imperfections go unnoticed but others may render the application essentially useless. For example, a 5\% accuracy reduction can have a negligible impact on user QoE, while 100ms extra delay can cause the expiration of results. The importance of QoE has been widely evidenced in the networking domain \cite{barakovic2013survey}. However, quantifying the benefits of QoE-awareness in on-thing inference is still an under-investigated topic. This paper studies \underline{O}n-thing \underline{I}nference \underline{C}ustomization (OIC) for DL-based applications. Our goal is to customize on-thing inference for individual users in a way that maximizes user perceived experience.

\subsection{Motivations}
OIC is motivated by the heterogeneity in IoT devices, DNN models, usage scenarios, and user preferences, which jointly affect user perceived QoE in a complicated manner.

\subsubsection{IoT Device Heterogeneity}
IoT devices in the market are extremely diverse in their computing capacities: some high-end devices have state-of-the-art CPUs along with dedicated graphic processing units (GPUs) and even purpose-built processing units (e.g., edge TensorFlow processing units \cite{edgetpu}) to speed up DNN inferences, while many others are powered by CPUs of several years old \cite{wu2019machine}. Consequently, there is no standard device model to optimize the DL inference for. The heterogeneity of device computing capacity results in a huge variability in service quality and user experience. Even with fine-tuned DNNs, the inference latency varies by a factor of 10+ across IoT devices \cite{wu2019machine}. 

\subsubsection{DNN Model Heterogeneity} 
There often exist various DNN architectures that can be used by the application developer to address a learning problem. For example, in \emph{image classification}, commonly-used DNN architectures include MobileNet \cite{howard2017mobilenets}, Inception \cite{szegedy2017inception}, NASNet \cite{zoph2018learning}, Yolo \cite{redmon2016you}, etc. The recent studies on DNN model compression \cite{cheng2018model,han2015deep,huang2018data}, e.g., network pruning, weight quantization, low-rank matrix approximation, and knowledge distillation, further provides more available DNN architectures that exhibit different trade-offs in a multi-dimension space of important metrics (e.g., size vs. accuracy vs. latency). Table \ref{table:dnn_hetero} shows performance metrics of several DNN examples provided by Tensorflow Lite \cite{tensorflowmodels}. We see that different DNNs require different computing resource (model size) and provide different inference performances (accuracy and latency). No single DNN can achieve optimality in all dimensions. 
\begin{table}[htb]
	\centering
	\caption{Performance of DNN models}
	\begin{tabular}{ m{2cm} |  m{1.3cm} m{1.3cm}  m{1.3cm}}
		\hline
		\textbf{Model} & \textbf{Size} & \textbf{Accuracy} & \textbf{Latency} \\\hline\hline
		Mobilenet v1 & \textbf{1.5MB} & 61.2\% & \textbf{3.6ms} \\
		Mobilenet v2 & 3.4MB & 70.8\% & 12ms \\
		Inception v2 & 11MB & 73.5\% & 59ms \\
		Inception v3 & 23MB & \textbf{77.5\%} & 148ms\\
		DenseNet & 43.6MB & 64.2\% & 195ms\\
		NasNet M & 21.4MB & 73.9\% & 56ms\\
		\hline
	\end{tabular}
	\label{table:dnn_hetero}
\end{table}

\subsubsection{Usage Scenario Heterogeneity} 
DNNs running on different IoT devices can be exposed to very different usage scenarios --- different locations, illumination conditions, time, etc. These all account to drastically different distributions of user input data, which can result in very different inference accuracies even for the same DNN model. For example, variations of image illumination can create intra-class variability in image classification problems \cite{taylor2018adaptive,windrim2016unsupervised} and degrade the performance of DL algorithms. Besides environmental factors, the IoT device status (e.g., CPU/memory usage and battery level) also affects the inference quality. For example, when the battery level is low, IoT devices may switch to the battery saving mode and decrease CPU frequency, which results in larger inference latency.

\subsubsection{User Preference Heterogeneity.} 
It should be noticed that users often have different sensitivity towards different performance metrics of DNN inferences. For example, some users are more energy-sensitive due to limited battery capacities, whereas others would like to trade energy consumption for lower inference latency and higher inference accuracy. In addition, the sensitivities of a user can also change under different circumstances, e.g., when the battery level becomes low, the user may become more energy-sensitive. 

\begin{figure*}[htb]
	\centering
	\includegraphics[angle =90,width=1\linewidth]{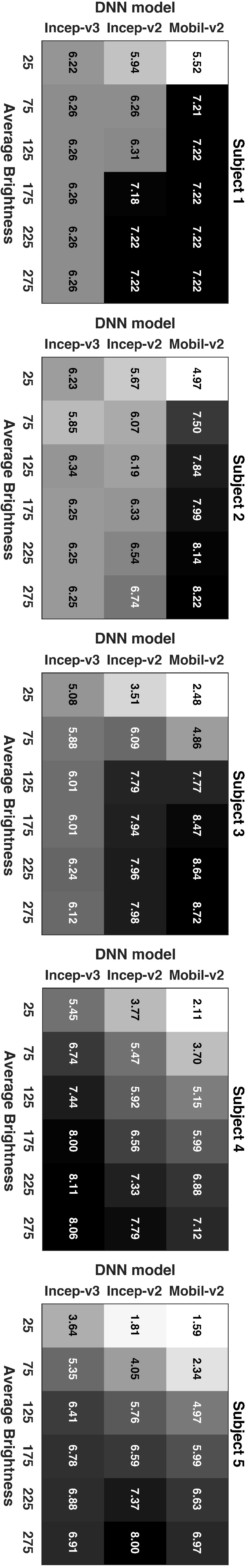}
	\caption{User QoE v.s. DNN models v.s. Environmental changes.}
	\label{fig:context_qoe}
\end{figure*}

A simple trial is carried out to support our claims. We surveyed five human subjects about their user experience when using the image classification application released by Tensorflow Lite \cite{tensorflowmodels}. The application classifies images captured by the device camera in real-time. Two kinds of devices, Motorola X$^4$ and OnePlus One, are tested: subject 1, 2, 3 use OnePlus One and subject 4, 5 use Motorola X$^4$. Different DNNs can be configured in the application to process classification tasks. We use three DNNs in the trial, MobileNet-v2, Inception-v2, and Inception-v3. At the beginning of each round, the application is randomly configured with a DNN, and the environment (e.g., the location of subjects, brightness, etc) also changes randomly. At the end of each round, we solicit the QoE from subjects. Fig. \ref{fig:context_qoe} depicts the average QoE of subjects using different DNNs under different environment (brightness as an example). The results validate some of our claims: 1) The device affects user QoE. We can see that subjects have similar QoE patterns if they use the same device. This because the device computing capacity determines the DNN inference delay which is one of the most key factors that affect user experience. 2) The user preferences are different. Even for subjects using the same device, their QoE patterns exhibit noticeable differences. For example, subject 4 reports a higher QoE using inception-v3 while subject 5 prefers inception-v2. 3) The usage scenario has a significant impact on user QoE. There is a drastic reflection on the QoE variation as we change the brightness of the environment. In particular, we can also observe that users and DNNs have different sensitivity to environmental changes, e.g., the QoE rating of subject 2 on inception-v2 is relatively stable with different brightness values while the QoE rating of subject 3 on inception-v2 varies significantly across brightness.

Motivated by these observations, OIC employs a model selection technique that aims to configure DL inference process with the best-fit DNN on-the-fly based on the computing capacity of IoT devices, usage scenarios, and user preferences. However, several challenges that need to be addressed before OIC can deliver what it is capable of.


\subsection{Challenges and Contributions}
The \textbf{first} challenge is the unknown and diverse user QoE patterns. The crux of OIC is finding an underlying mapping for each DNN that maps the device information, user preference, and usage scenario to the user QoE pattern. If these mappings are available at hand, then the best-fit DNN can be easily identified. Unfortunately, such mappings are unknown a priori, and therefore a learning mechanism is necessary to acquire the user QoE patterns. In particular, OIC will require the learning method to have a good generalization ability as user QoE patterns usually exhibit significant variability due to the heterogeneity in IoT devices and user preferences. The \textbf{second} challenge is the user-perceived inconvenience when learning user QoE patterns online. To customize the on-thing inference for a particular user, we will need QoE data that precisely reflects the user's satisfaction about DNNs under various usage scenarios. Such QoE data can only be collected from users while they are using the application. This is often done by on-screen pop-ups that solicit the user about his/her experience. Such QoE solicitation strategy is commonly adopted, e.g., Skype asks a “How was your call quality?” question when a call is ended. Being users ourselves, we know that frequently receiving experience surveys can be annoying. Also, QoE solicitations may incur additional cost to application developers because incentive mechanisms \cite{carson2007incentive} are often applied to motivate the QoE response. Therefore, a key designing topic of OIC is to keep the solicitation inconvenience minimal during online learning. The \textbf{third} challenge is to guarantee the practicality of OIC. Despite the efficacy, other practical issues should also be considered to successfully deliver functionalities of OIC in practice. For example, the computational complexity of OIC must be low enough to work on IoT devices. OIC is also expected to take effect quickly and hence the learning process should be kept short. In addition, users may provide aggregated QoEs that reflect their experience for multiple DNNs used before rather than an individual DNN. All these practical issues should be taken care of when designing OIC.

In this paper, we design a novel framework to address the above challenges. Fig. \ref{fig:overflow} depicts a block diagram of the proposed OIC method. It consists of two components: a \emph{DNN Selection Module} that learns a QoE predictor for guiding the DNN selection; and a \emph{Feedback Solicitation Scheme} (FSS) that determines when to pop up experience surveys. The contribution of this paper is summarized as follows:    
\begin{figure}[tb]
	\centering
	\includegraphics[width = 0.9\linewidth]{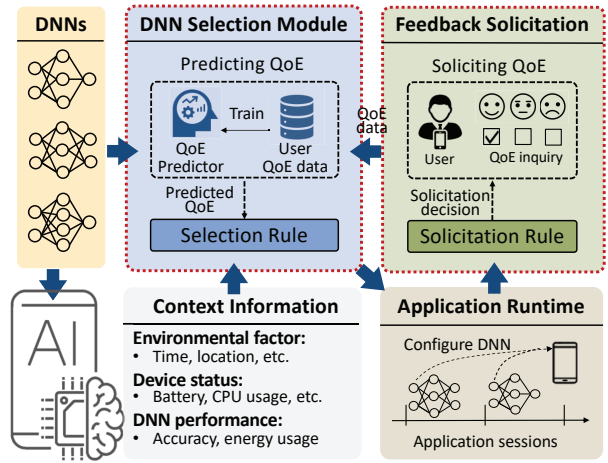}
	\caption{Block diagram of OIC.}
	\label{fig:overflow}
\end{figure}

1) We utilize online learning techniques to build the DNN selection module. A novel contextual multi-armed bandit algorithm called NeuralUCB \cite{zhou2019neural} is employed to learn user QoE based on the side-information of DNNs and usage scenarios. The predicted QoE pattern is then used to identify the best-fit DNN for the application. A salient feature of NeuralUCB is that it does not assume a certain distribution of QoE data and hence, attains generality across different users and devices. We further incorporate knowledge transferring techniques in NeuralUCB to speed up the customization process.

2) Feedback solicitation schemes (FSS) are investigated to reduce the solicitation inconvenience incurred by online learning. FSS aims to reduce the number of QoE solicitations without harming the performance of NeuralUCB. The designed FSS strikes a balance between learning efficiency and solicitation cost. It cuts the number solicitations to $T^{2/3}$ (a 90\% reduction for $T = 1000$, where $T$ is the number of solicitations required without FSS) while keeping asymptotic optimality of NeuralUCB with sublinear performance loss $\mathcal{O}(T^{2/3})$. 

3) A learning strategy is further designed to apply OIC with the aggregated QoE where QoE ratings solicited from users reflect their experience over multiple DNNs used during a time span. The key to addressing the aggregated QoE is a feedback refinement approach that estimates individualized QoEs (i.e., QoEs for individual DNNs) from aggregated QoEs with the assistance of QoE predictor. The feedback refinement approach is a flexible add-on and also handles a mix of aggregated and non-aggregated QoEs.

4) The proposed method is evaluated on both numerical and real-world data. We collect context information and user QoE from human subjects when using a \emph{image classification application} \cite{tensorflowmodels}. The results show that OIC can learn user QoE patterns with low solicitation costs and dramatically improve the user experience. 

The rest of the paper is organized as follows.  Section \ref{sec:related_work} discusses related works. Section \ref{sec:online_learning} designs the DNN selection module. Section \ref{sec:fss} develops feedback solicitation schemes. Section \ref{sec:aggregated} designs a feedback refinement approach for the aggregated QoE. Section \ref{sec:experiment} carries out experiments and evaluations, followed by conclusions in Section \ref{sec:conclusion}.

\section{Related Works}\label{sec:related_work}
\textbf{Deep Learning for IoT Device.}
The past few years have seen a surge in the investigation of DL techniques for IoT and embedded platforms. Promising results are appearing across many domains including hardware, learning algorithms, and tools. Many CPU/GPU vendors are developing new processors for supporting tablets/smartphones to run DL applications. A notable example is Apple Bionic chips \cite{sima2018apple} which include dedicated neural network hardware that Apple calls a "Neural Engine". Various algorithmic techniques also have been proposed. For example, DNN compression methods \cite{xie2019source, han2017ese} prune large-scale DNN models into small DNNs that can be easily implemented on mobile and IoT devices. New tools and libraries, e.g., Tensorflow Lite \cite{lite2017android} and Core ML \cite{coreML} have been proposed to address the specific needs of resource-constrained devices. 

\textbf{DNN Selection.}
DNN selection aims to identify the best-fit DNN from a set of candidate DNNs that are trained with different architectures, parameters, and training data. Unlike DNN selections for general purposes that focus on accuracy and inference delay, DNN selections for IoT devices should take into account other important factors, e.g., computing resource requirement and energy consumption, due to the limited computing capacity and battery on mobile devices. Recent works consider the special needs of mobile devices when designing DNN selection schemes. For example, authors in \cite{park2015big} designed a big/little DNN selection method that considers the energy consumption and accuracy of DNNs. The work \cite{stamoulis2018designing} considered energy consumption, accuracy, and communication constraints posed by mobile devices. All these works select DNNs by optimizing QoS with a specific objective function. By contrast, our DNN selection module is QoE-driven and learns user QoE patterns online. The most related work is probably \cite{lu2019automating}, where the authors also studied a model selection scheme for QoE improvement. Comparing with it, our work has two stark differences. First, the model selection in \cite{lu2019automating} is performed on Cloud, which finds an appropriate model for a user based on its device features. By contrast, our on-thing inference customization method runs model selection on IoT devices, aiming to switch models adaptively according to the user preference and usage scenario. Second, \cite{lu2019automating} used only static device information to guide the model selection while our work encompasses a variety of dynamic states (e.g., task features, device status, and environment changes) and discover complicated QoE pattern for individual users under various states. 

\textbf{Contextual Multi-armed Bandit.}
The contextual bandit problem has been extensively studied in machine learning \cite{langford2007epoch,bubeck2012regret}. The most studied model in the literature is linear contextual bandits \cite{rusmevichientong2010linearly,dani2008stochastic,abe2003reinforcement,auer2002using}, which assume that the expected reward at each round is linear in the context vector. While successful in theory, the linear-reward assumption it makes often fails to hold in practice, which motivates the study of nonlinear or nonparametric contextual bandits \cite{filippi2010parametric,bubeck2011x,valko2013finite,srinivas2009gaussian}. However, they still require fairly restrictive assumptions on the reward function. For instance, \cite{filippi2010parametric} makes a generalized linear model assumption on the reward, \cite{bubeck2011x} requires it to have a Lipschitz continuous property in a proper metric space, and \cite{valko2013finite} assumes the reward function belongs to some Reproducing Kernel Hilbert Space (RKHS). To overcome the above shortcomings, we employ the recently proposed neural contextual bandit, NeuralUCB \cite{zhou2019neural}, as the online learning tool. NeuralUCB utilizes the strong representation power of neural networks to learn nonlinear mappings from context information to the user QoE and follows a UCB strategy \cite{auer2002using} for exploration.

\section{Automated Customization of On-thing Inference}\label{sec:online_learning}
\subsection{System Overview}
\subsubsection{Pre-configuration}
OIC is designed as a plug-in module for IoT applications, and it is deployed on the user device along with the application installation. The application will download multiple DNN models (hereinafter referred to as candidate DNNs), denoted by $\mathcal{M}=\{1,\dots,M\}$, from the developer's cloud. The number of candidate DNNs $M$ should not be too large to avoid excessive storage usage on the user device. The candidate DNNs have different performances in terms of accuracy, latency, energy consumption, etc. The set of candidate DNNs $\mathcal{M}$ defines the action space of DNN selection. In a broader concept, the action space of DNN selection are not necessarily independent DNNs, it can also be different operation points of one DNN model. For example, using the \emph{early exits} technique (e.g., BranchyNet \cite{teerapittayanon2016branchynet}), a DNN can exit earlier (to save time) with a lower inference accuracy, and the set of available exit points (with different accuracy and delay levels) becomes the action space. 

\subsubsection{Operations of On-thing Inference Customization}
The operation timeline of OIC is discretized by \emph{application sessions}, denoted by $\mathcal{T} = \{1,2,\dots, T\}$. An application session begins when a user starts the application and ends when the application exits. The DNN selection module runs in the application loading phase during which it selects a DNN and configures it in the application. Note that the application loading phase originally exists for application initialization, object pooling, server connection, etc. Our method does not prolong this loading phase as we will show later in the experiment that the run-time of our method is less than 2ms. We for now assume that the selected DNN does not change during a session. An extended scene will be considered in Section \ref{sec:aggregated} where the selected DNN changes during a session. The QoE solicitation happens at the end of application sessions. This is done by popping up an on-screen survey that inquires about the user experience during the application session. It is possible that users are not willing to complete the QoE survey questions. In this case, the operator may apply incentive mechanisms \cite{carson2007incentive} to motivate the QoE response. The feedback solicitation schemed designed later (in Section \ref{sec:fss}) will consider the cost of QoE solicitations due to the incentive payment and the inconvenience brought to users. The collected QoE will be stored in a database for learning user QoE patterns later.
\begin{figure}[htb]
	\centering
	\includegraphics[width = 0.95\linewidth]{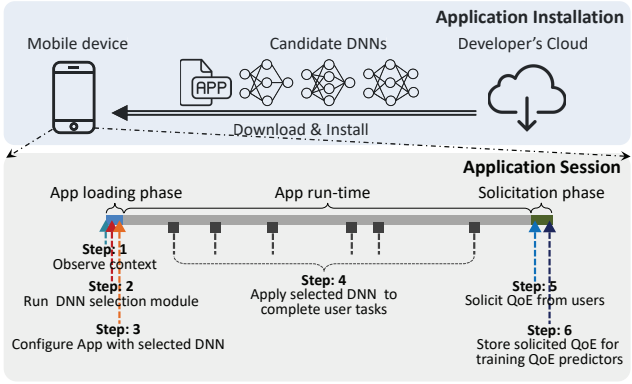}
	\caption{Operations during an application session.}
	\label{fig:operational_timeline}
\end{figure}

\subsubsection{QoE Predictor}
A QoE predictor is built to guide the DNN selection. It establishes a mapping from the context of DNN and usage scenario to user QoE. Examples of context information includes: 1) \emph{environmental context} which can be acquired via sensors equipped with the IoT device, e.g., the device location acquired via GPS signals and the illumination acquired via the ambient light sensor; 2) \emph{device status}, e.g., battery level, CPU usage, and memory usage, which can be inspected by calling Java APIs; 3) \emph{DNN performance}, e.g., size, accuracy, and latency of DNNs. The statistics of DNN performances are available on the TensorFlow Lite website \cite{tensorflowmodels}. Note that these statistics are \emph{nominal} because they are measured on standard device which may be different from the user device, and on a standard dataset which may be different from the user input data. Therefore, it is possible that the performance of a DNN on the user device can be very different from the nominal performance. The context for a DNN $m\in\mathcal{M}$ in application session $t$ is denoted by $x_{t,m}$. We let $\bm{x}_t = \{x_{t,m}\}_{m\in\mathcal{M}}$ collect context of all candidate DNNs. The QoE predictor predicts the user QoE for all candidate DNNs $\hat{\bm{r}}_t = \{\hat{r}_{t,m}\}_{m\in\mathcal{M}}$, where $\hat{r}_{t,m}$ is the predicted QoE for DNN $m$ given context $x_{t,m}$. If QoE predictions are accurate, then the best-fit DNN for session $t$ can be easily determined by \emph{greedy policy} $m_t = \argmax\nolimits_{m\in\mathcal{M}} ~ \hat{r}_{t,m}$. 

Fig. \ref{fig:operational_timeline} depicts the timeline of operations during an application session. Now a fundamental problem for OIC becomes learning an accurate QoE predictor for candidate DNNs. In the sequential, we utilize a multi-armed bandit (MAB) algorithm to learn QoE predictors in an online fashion.

\subsection{Learning QoE Predictors}\label{sec:neural_bandit}
Note that the greedy policy for selecting the best-fit DNN is only plausible when the QoE predictor give accurate enough results, otherwise, the selected DNN may deliver much worse QoE than the expectation. Obtaining a high-quality QoE predictor requires adequate collections of context-QoE data for candidate DNNs. Note that the user QoE for a DNN can be solicited only when the DNN is used during a session, and therefore the purpose of selecting a DNN can be either \emph{exploration}, i.e., to collect user QoE delivered by the selected DNN for better QoE prediction in the future, or \emph{exploitation}, i.e., to select the DNN that is expected to deliver the highest QoE to the user. An important designing goal in MAB problems is balancing the tradeoff between exploration and exploitation.   

Because the user QoE is relevant to the context, we assume the QoE value $r_t$ in session $t$ is randomly sampled from an unknown distribution $\mathcal{R}$ parameterized by the context of selected DNN $x_{t,m_t}$, i.e., $r_t \sim \mathcal{R}(x_{t,m_t})$. Due to the heterogeneity in user preferences, we do not restrict distribution $\mathcal{R}(x_{t,m_t})$ to a certain type. This requires our MAB algorithm to attain good generalization ability for handling any possible QoE distributions the users may have. 

\begin{algorithm}[tb]
	\caption{NeuralUCB for OIC} \label{alg:Neural-UCB}
	\begin{algorithmic}[1]
		\State \textbf{Input}: time horizon $T$, algorithm parameter $\gamma$, the number of nodes in QPN hidden layers $h$.
		\State \textbf{Initialization}: Randomly initialize $\bm{\theta}$, initialize $\bm{Z_0} = \bm{I}$
		\For {application session $t=1,\dots,T$}
		\For {each DNN $m\in\mathcal{M}$}
		\State Observe the context of DNN $m$, $x_{t,m}$ 
		\State Predict QoE for DNN $m$: $\hat{r}_{t,m} \gets \hat{r}(x_{t,m};\bm{\theta})$
		\parState {Computing the gradient of QPN parameter $\bm{\theta}$ at $x_{t,m}$: $\bm{g}_{t,m} \gets \nabla_{\theta}~\hat{r}(x_{t,m};\bm{\theta})$}
		\EndFor
		
		\State Compute $u_{t,m} \gets \hat{r}_{t,m} + \gamma\sqrt{\bm{g}_{t,m}^\top\bm{Z}_{t-1}^{-1}\bm{g}_{t,m}/h}, \forall m$\label{line:selection1} 
		\State Select the DNN $m_t = \argmax_{m\in \mathcal{M}}u_{t,m}$ \label{line:selection2} 
		\parState {Solicit QoE $r_{t,m_t}$ from the user and store the context-QoE data $(x_{t,m_t}, r_{t,m_t})$ in data set $\mathcal{X}$}
		\State Update QPN: $\bm{\theta} \gets \texttt{TrainQPN} \left(\mathcal{X}\right)$ \Comment{Algorithm \ref{alg:TrainNN}}
		\State Compute $\bm{Z}_t \gets \bm{Z}_{t-1}+\bm{g}_{t,m_t}\bm{g}^\top_{t,m_t}/h$;
		\EndFor
	\end{algorithmic}
\end{algorithm}
\begin{algorithm}[tb]
	\caption{Subroutine: \texttt{TrainQPN}$\left(\mathcal{X}\right)$} \label{alg:TrainNN}
	\begin{algorithmic}[1]
		\State \textbf{Input}: learning rate $\eta$, context-QoE dataset $\mathcal{X}$, number of gradient descent update steps $J$;
		\State Define $L(\bm{\theta})=\sum_{(x_{\tau,m_\tau},r_{\tau,m_\tau})\in \mathcal{X}} \big(\hat{r}(x_{\tau,m_\tau};\bm{\theta})-r_{\tau,m_\tau}\big)^2$ 
		\For {$j=1,\dots,J-1$}
		\State $\bm{\theta}_{j+1} = \bm{\theta}_j-\eta\nabla L(\bm{\theta}_j)$
		\EndFor
		\State \Return $\bm{\theta}_J$
	\end{algorithmic}
\end{algorithm}
\subsubsection{NeuralUCB}
Our work employs a novel MAB algorithm called \emph{NeuralUCB} \cite{zhou2019neural}. NeuralUCB constructs neural networks to approximate user QoE patterns. It has the capability of identifying and representing general dependencies in the data without a priori specifying which specific form of distribution to look for. The neural network can act either as a classification model for label-form QoEs (e.g., \emph{"satisfactory"} or \emph{"not satisfactory"}) or a regression model for rating-form QoEs (e.g., $r_t \in [0,1]$). The neural network constructed in NeuralUCB is referred to as \textbf{QoE Predicting Network (QPN)} and works as a QoE predictor in DNN selection module. The parameter vector of QPN is denoted by $\bm{\theta}$. The input to QPN is the context $x_{t,m}$ of a DNN $m\in\mathcal{M}$ and the output of QPN is the predicted QoE delivered by DNN $m$, denoted by $\hat{r}(x_{t,m},\bm{\theta})$. The goal of NeuralUCB is to maximize the cumulative user QoE $\sum^T_{t=1} r_{t,m_t}$ by selecting an appropriate DNN $m_t$ in each session $t$.

The pseudocode of NeuralUCB for OIC is given in Algorithm \ref{alg:Neural-UCB}. In each session, a DNN is selected based on $u_{t,m} : =  \hat{r}(x_{t,m},\bm{\theta})   + \gamma\sqrt{\bm{g}_{t,m}^\top \bm{Z}_{t-1}^{-1} \bm{g}_{t,m}/h}$ (Line \ref{line:selection1} and \ref{line:selection2}). The first term of $u_{t,m}$ is the user QoE predicted by QPN and the second term of $u_{t,m}$ characterizes the uncertainty of the prediction. If the QoE prediction exhibits large uncertainty for DNN $m$, then NeuralUCB has a tendency of selecting DNN $m$ in order to collect its context-QoE data for improving the prediction performance. Otherwise, $u_{t,m}$ is dominated by the predicted QoE value and NeuralUCB selects the DNN that is expected to deliver the highest QoE. The parameter $\gamma$ is used to adjust the importance of exploration and exploitation.

\noindent\textbf{Remark on complexity of NeuralUCB.}
Note that OIC runs on IoT devices, it should be assured that the computational complexity of NeuralUCB is acceptable to resource-constrained IoT devices. The computational cost of NeuralUCB lies mainly in training and running QPN. Later in the experiment, we will show that a simple neural network (3-layer fully connected network with 8, 16, 8 nodes) can achieve good performances. Therefore, the space and time complexity of NeuralUCB is low.

\subsubsection{Speed Up Learning with Knowledge Transferring}
The standard NeuralUCB algorithm learns without any a priori knowledge. It initializes QPNs with random parameters and gradually trains the QPN to approximate the QoE pattern of a particular user. However, certain general knowledge of user QoE patterns is actually available. Although QoE patterns differ across users, the users still share some preferences over certain performance metrics, e.g., all users will prefer higher inference accuracy, lower inference delay, and less energy consumption. In observation of this, we incorporate knowledge transferring \cite{konidaris2006autonomous} (a.k.a. transfer learning \cite{torrey2010transfer}) into NeuralUCB to speed up the customization process. With knowledge transferring, a pre-trained QPN is used in the initialization of NeuralUCB. The pre-trained QPN contains the general knowledge on user QoE patterns and is provided by application developers. For example, application developers can collect QoEs of DNNs from a group of test users on standard IoT devices, and train a QPN that works generally well for all tested users. During online learning, the pre-trained QPN is retrained to approximate the QoE pattern of a particular user via NeuralUCB. Using knowledge transferring dramatically accelerates the customization process.

\subsubsection{Performance guarantee of NeuralUCB}
The performance of NeuralUCB is measured by \emph{regret} which defines the performance loss compared to an Oracle that selects the optimal DNN, $m_t^{*} = \arg\max_{m\in \mathcal{M}} r_{t,m}$, in each session $t$. The regret is formally defined as:
\begin{equation} \label{eq:regret}
R_T = \mathbb{E}\left[\sum\nolimits_{t=1}^T  r_{t,m_t^*} - r_{t,m_t} \right]
\end{equation} 
Next, we state a lemma that provides a performance guarantee of NeuralUCB in terms of regret upper bound. 

\begin{lemma} \label{lemma:neuralucb_regret_bound} (Regret upper bound of NeuralUCB). The regret upper bound of NeuralUCB is $\mathcal{O}(R_T) = \mathcal{O}(\sqrt{T})$, where $T$ is the total number of application sessions.
\end{lemma}
The proof for Lemma \ref{lemma:neuralucb_regret_bound} can be found in \cite{zhou2019neural}. It gives a sublinear regret $\mathcal{O}(\sqrt{T})$, meaning that NeuralUCB is \textbf{asymptotically optimal} compared to Oracle. Standard NeuralUCB solicits QoE feedback from users after every session. The users may feel these survey questions annoying if they appear constantly. In certain cases, the application developer also needs incentive schemes to motivate the user response. Therefore, a feedback solicitation scheme needs to be designed to control the solicitation cost (e.g., incentive payment and user inconvenience) of the online learning algorithm.

\section{Feedback Solicitation}\label{sec:fss}
The goal of feedback solicitation scheme (FSS) is to reduce the number of QoE solicitations while keeping the learning performance of NeuralUCB. Let $Q(t)$ denote the event of QoE solicitation in session $t$, then the solicitation cost is defined as $\lambda \cdot \mathbf{1}\{Q(t)= \text{True}\}$, where $\lambda$ is the unit cost of one executed solicitation. We define the \emph{reward} in session $t$ as $\tilde{r}_{t,m} = r_{t,m} - \lambda\cdot \mathbf{1}\{Q(t)= \text{True}\}$. The performance of FSS is measured by \emph{modified regret} (m-regret):
\begin{equation*}
\begin{split}
\tilde{R}_T & = \mathbb{E}\left[\sum\nolimits_{t=1}^T \tilde{r}_{t,m_t^*}-\tilde{r}_{t,m_t}\right] \\
& =\mathbb{E}\left[\sum\nolimits_{t=1}^T r_{t,m_t^*}-r_{t,m_t} + \lambda\cdot \mathbf{1}\{Q(t)= \text{True}\} \right] \\
& = \underbrace{\mathbb{E}\left[\sum\nolimits_{t=1}^T r_{t,m_t^*}-r_{t,m_t}  \right]}_{\text{Learning regret}} + \underbrace{\lambda\cdot\sum\nolimits_{t=1}^T\mathbf{1}\{Q(t)= \text{True}\}}_{\text{Solicitation cost}} 
\end{split}
\end{equation*}
where $\tilde{r}_{t,m_t^*}$ is the reward achieved by Oracle in session $t$ and its value equals $r_{t,m_t^*}$ because Oracle knows the user QoE pattern and does not need QoE solicitations. The m-regret consists of two parts. The first term is the same to the standard regret in \eqref{eq:regret}, which measures the learning performance and we call it \emph{learning regret}. The second term measures the \emph{solicitation cost}. Intuitively, with more frequent QoE solicitations, we are able to learn the user QoE pattern faster, thereby reducing the learning regret. However, frequently soliciting QoEs from users causes a high solicitation cost. Consider the standard learning process of NeuralUCB, although it guarantees a sublinear learning regret $\mathcal{O}(\sqrt{T})$, its solicitation cost increases linearly with the number of application sessions $\mathcal{O}(T)$, resulting in a linear m-regret.

\subsection{Feedback Solicitation Scheme}
We design feedback solicitation scheme (FSS) to strike a balance between the learning regret and solicitation cost. In the following proposition, we directly give out our basic FSS design for NeuralUCB.

\begin{proposition}[Feedback Solicitation Scheme]\label{prop:fss}
	Let $T$ be the total number of application sessions, the frequency of QoE solicitation for NeuralUCB should be set to $T^{-1/3}$ for minimizing m-regret. 
\end{proposition}

\begin{figure}[tb]
	\centering
	\includegraphics[width = 0.95\linewidth]{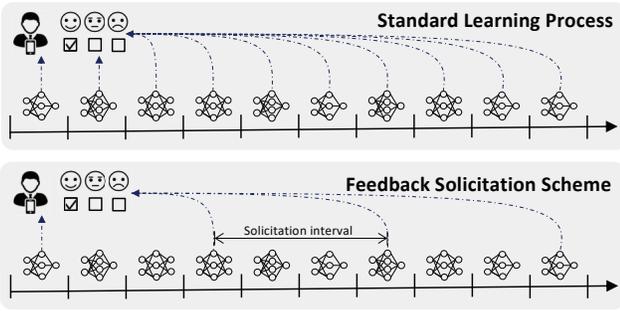}
	\caption{Illustration of feedback solicitation scheme.}
	\label{fig:illu_fss}
\end{figure}

Proposition \ref{prop:fss} indicates that the QoE solicitation is performed periodically with a fixed frequency $T^{-1/3}$, i.e., QoE solicitation happens every $T^{1/3}$ sessions. If NeuralUCB runs $T=1000$ sessions, OIC only solicits QoE for $1000^{2/3} = 100$ times, i.e., a 90\% reduction of solicitation cost. Fig. \ref{fig:illu_fss} illustrates the designed FSS. The solicitation frequency is determined according to the regret upper bound $\mathcal{O}(\sqrt{T})$ (Lemma \ref{lemma:neuralucb_regret_bound}) of NeuralUCB. FSS aims to guarantee a sublinear m-regret such that NeuralUCB is still asymptotic optimal with reduced QoE solicitations. The theorem below gives an upper bound of m-regret of NeuralUCB with FSS. 
\begin{theorem}[m-regret of NeuralUCB with FSS] \label{theo:bound_m_regret}
	Suppose NeuralUCB is applied with FSS designed in Proposition \ref{prop:fss}. For any time horizon $T$, the upper bound of its m-regret is $\mathcal{O}(\tilde{R}_T) = \mathcal{O}(T^{2/3})$. 
\end{theorem}
\begin{proof}
	See in Appendix \ref{proof:theo:bound_m_regret}.
\end{proof}

The above theorem indicates a sublinear upper bound $\mathcal{O}(T^{2/3})$ for m-regret of NeuralUCB-FSS. This means that our method is able to learn an asymptotically optimal DNN selection policy while keeping the number of QoE solicitations as low as $T^{2/3}$. Note that the regret upper bound given in Theorem \ref{theo:bound_m_regret} is for the worst case with an arbitrarily bad initial QPN. With knowledge transferring, the regret incurred by NeuralUCB is actually much lower, which is verified by our experimental results in Section \ref{sec:experiment}. 

\subsection{FSS with Unknown Time Horizon}
One limitation of FSS is that it needs to know the time horizon $T$ in advance to determine the solicitation frequency. However, the time horizon is often unknown in practice, in which case FSS cannot be applied directly. In the sequential, we show a variant of FSS called FSS-UT that is able to work with an unknown time horizon $T$. FSS-UT keeps a counter $c$ that counts the number of QoE solicitations executed up to application session $t$. If the counter satisfies $c \leq t^{1-\alpha}$, then a QoE solicitation will be executed in session $t$. Fig. \ref{fig:illu_fss_ut} illustrates the designed FSS-UT. FSS-UT starts with a relatively small solicitation interval to ramp up QoE data for training the QoE predictor, and then gradually increase the solicitation interval over time. FSS-UT guarantees that $t^{1-\alpha}$ QoE solicitations will be executed up to session $t$. If we let $\alpha = 1/3$, the number of QoE solicitations executed by FSS-UT in $T$ sessions (i.e., $T^{2/3}$) is the same as that of FSS (Proposition \ref{prop:fss}), and in this case, the solicitation cost of FSS-UT is the same as FSS. The theorem below provides a performance guarantee of FSS-UT in terms of m-regret.

\begin{theorem}[m-regret of NeuralUCB with FSS-UT]\label{theo:regret_fss_ut}
	Suppose NeuralUCB is applied with FSS-UT using parameter $\alpha$. For any time horizon $T$, the upper bound of m-regret is $\mathcal{O}(\tilde{R}_T) = \mathcal{O}(T^z)$ with $z = \max\left\{\frac{\alpha+1}{2},1-\alpha\right\}$. 
\end{theorem}
\begin{proof}
	see in Appendix \ref{proof:theo:regret_fss_ut}.
\end{proof}
In the above theorem, $\mathcal{O}(T^{(\alpha+1)/2})$ is the order of learning regret, $\mathcal{O}(T^{1-\alpha})$ is the order of solicitation cost, and the balance is achieved at $\alpha = 1/3$. Therefore, the lowest order of m-regret is $\mathcal{O}(T^{2/3})$. 
\begin{figure}[tb]
	\centering
	\includegraphics[width = 0.95\linewidth]{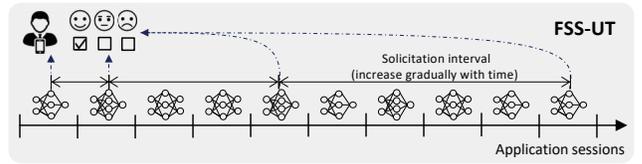}
	\caption{Illustration of FSS-UT.}
	\label{fig:illu_fss_ut}
\end{figure}

\section{Learning with Aggregated QoE}\label{sec:aggregated}
The previous sections have shown the online learning algorithm, NeuralUCB, and feedback solicitation schemes for solicitation cost reduction. However, the proposed methods rely on two underlying assumptions: 1) only one DNN is selected in each application session and the selected DNN is used during the whole session; 2) the user QoE, if solicited, accurately reflects the QoE of a single DNN used in the session. In practice, an application session can be quite long and the context may change during a session. For example, while using the application, the user may move to different locations and the device battery level will decrease. This requires the application to reconfigure its DNN adaptively according to context changes. In this case, a QoE feedback solicited at the end of an application session may reflect the user experience over all (potentially different) DNNs used in that session, and we call this kind of QoE value aggregated QoE. Fig. \ref{fig:aggregated_illu} gives an illustration of aggregated QoEs. The standard NeuralUCB presented previously cannot handle the aggregated QoE. In this section, we develop a feedback refinement approach as a subroutine of NeuralUCB to deal with aggregated QoEs. 

\begin{figure}[htb]
	\centering
	\includegraphics[width = 0.95\linewidth]{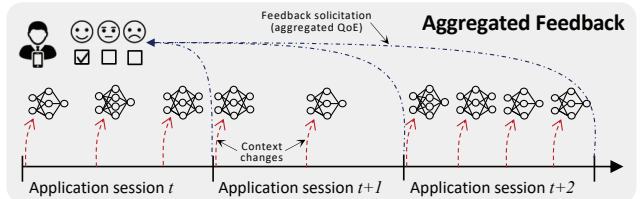}
	\caption{Illustration of aggregated QoEs.}
	\label{fig:aggregated_illu}
\end{figure}

The goal of feedback refinement is to estimate QoEs for individual DNNs (referred to as individualized QoE) based on aggregated QoEs collected from users. A consideration here is that the estimated individualized QoEs are subjected to the constraint that a combination (e.g. a weighted sum) of individualized QoEs should equal the aggregated QoE. 

Before presenting the feedback refinement approach, we need a modelling of aggregated QoEs. In application session $t$, we let $K_t$ denote the times of context change. The set of contexts appeared in session $t$ is collected in $\bm{x}^{\texttt{AG}}_t = \left\{\bm{x}^{k}_t\right\}^{K_t}_{k=1}$ where  $\bm{x}^{k}_t = \{x^{k}_{t,m}\}_{m\in\mathcal{M}}$ is the $k$-th context of candidate DNNs. Let $m^k_t$ denote the selected DNN given context $\bm{x}^{k}_t$. The set of DNNs selected in session $t$ is denoted by $\bm{m}^\texttt{AG}_t = \{m^k_t\}^{K_t}_{k = 1}$. The aggregated QoE collected at the end of session $t$ is denoted by $r^\texttt{AG}_t$. The contexts $\bm{x}^\texttt{AG}_t$, selected DNNs $\bm{m}^\texttt{AG}_t$, and aggregated QoE $r^\texttt{AG}_t$ are stored in memory $\mathcal{
	X}^\texttt{AG} \gets \mathcal{X}^\texttt{AG} \cup (\bm{x}^\texttt{AG}_t, \bm{m}^\texttt{AG}_t, r^\texttt{AG}_t)$.

The feedback refinement approach is a two-step loop towards convergence: 1) individualized QoEs are first estimated with the assistance of the QoE predictor, and then 2) estimated QoEs will be used in return to update the QoE predictor. For example, suppose OIC need to update QPN in session $t$ with the collected aggregated QoEs, the feedback refinement approach iterates between two steps:

\noindent \textbf{Step 1:} Given memory $\mathcal{X}$, we estimate individualized QoEs with the current QPN $\hat{r}(\cdot;\bm{\theta})$. For each sample $(\bm{x}^\texttt{AG}_\tau, \bm{m}^\texttt{AG}_\tau, r^\texttt{AG}_\tau) \in \mathcal{X}^\texttt{AG}$ with some $\tau \leq t$, we compute its group residual:
\begin{align}
\delta_{\tau} = r^\texttt{AG}_{\tau} - \frac{1}{K_{\tau}}\sum\nolimits_{k=1}^{K_{\tau}}\hat{r}\left(x^k_{\tau,m^k_\tau};\bm{\theta}\right).
\end{align}

Then, the individualized QoE for DNN $m^k_{\tau}$ is estimated by
\begin{align}\label{eq:individual_r_est}
r^k_{\tau} = \hat{r}\left(x^k_{\tau,m^k_\tau};\bm{\theta}\right) + \delta_{\tau}.
\end{align}

It can be easily verified that the average of individualized QoEs estimated by \eqref{eq:individual_r_est} equals the aggregated QoE.

\noindent \textbf{Step 2:} The estimated individualized QoEs $r^k_\tau$ and context $x^k_{\tau,m^k_\tau}$ are stored in dataset $\mathcal{X}$ for updating the QPN with \texttt{TrainQPN} (Algorithm \ref{alg:TrainNN}). The only difference is that we use estimated individualized QoEs from Step 1 instead of ground-truth QoEs collected from users. The updated QPN will be used in the next iteration.

\begin{algorithm}
	\caption{Feedback Refinement Approach} \label{alg:aggregated feedback}
	\begin{algorithmic}[1]
		\While {group residual does not converge} 
		\For{$(\bm{x}^\texttt{AG}_\tau, \bm{m}^\texttt{AG}_\tau, r^\texttt{AG}_\tau) \in \mathcal{X}^\texttt{AG}$}
		\State Compute residual $\delta_{\tau} = r^\texttt{AG}_{\tau} - \frac{1}{K_{\tau}}\sum\limits_{k=1}^{K_{\tau}}\hat{r}(x^k_{\tau,m^k_\tau};\bm{\theta})$
		
		\State Get individualized QoE: $r^k_{\tau} = \hat{r}(x^k_{\tau,m^k_\tau};\bm{\theta}) + \delta_{\tau}$
		\State Store $(x^k_{\tau,m^k_\tau}, r^k_\tau)$ in $\mathcal{X}$
		\EndFor
		\State $\bm{\theta}\gets \texttt{TrainQPN}(\mathcal{X})$
		\EndWhile
	\end{algorithmic}
\end{algorithm}

The convergence of the above process is proven in \cite{bhowmik2019estimagg}. The feedback refinement approach can be easily extended to a scenario where collected QoEs contain both aggregated QoEs and non-aggregated QoEs. In this case, aggregated QoEs are stored in $\mathcal{X}^\texttt{AG}$ and will be used to estimate individualized QoEs. The non-aggregated QoEs are stored in $\mathcal{X}$ and are directly used to train QPN. Intuitively, if the fraction of aggregated QoEs is smaller, OIC can achieve better performance. In addition, FSS designed in Section \ref{sec:fss} can also be applied in the scenario of aggregated QoE to reduce the solicitation cost.

\section{Experiments and Results} \label{sec:experiment}
Due to the lack of available large-scale real-world datasets, we first run experiments on synthetic data to show the general performance of our method. We also collect real-world data from real users, which will be used to show the efficacy of customization.

\subsection{Numerical Experiments}
We first construct a mapping from context to user QoE. The user QoE is formulated as a weighted sum of \emph{accuracy} and \emph{delay}: the QoE of user $i$ is defined as $\texttt{QoE}_i = w^\texttt{a}_i a_i + w^\texttt{d}_i d_i$, where $w^\texttt{a}_i$ is the weight for the service accuracy $a_i$ and $w^\texttt{d}_i$ is the weight for the service delay. These weights are different for different users to capture the heterogeneous user preference. Both service performance (accuracy $a_i$ and delay $d_i$) and the user weights ($w^\texttt{a}_i$ and $w^\texttt{d}_i$) depend on the context information that includes \emph{brightness}, \emph{location}, \emph{CPU temperature}, \emph{time}, and \emph{battery level}, \emph{DNN nominal accuracy}, and \emph{DNN nominal delay}. Three DNN models are considered: MobileNet-v2 (nominal accuracy 70.8\%, nominal delay 12ms), Inception-v2 (nominal accuracy 73.5\%, nominal delay 59ms), and Inception-v3 (nominal accuracy 77.5\%, nominal delay 148ms). The nominal accuracy and nominal delay are from TensorFlow Lite website \cite{tensorflowmodels}, which are measured on ILSVRC 2012 image classification tasks \cite{ILSVRC} with Pixel 3 smartphone (Android 10). Note that the nominal performance can be different from the in-use performance because the user device and classification tasks are different.

The inference accuracy for user $i$ is determined by the nominal accuracy ($\texttt{nacc}_m$) of selected DNN $m$ and the context ambient brightness ($\texttt{brt}$): $a_i = \texttt{brt}^{-2\cdot \texttt{nacc}_m}/\texttt{brt}^{-2}$. This function indicates that a DNN with higher nominal accuracy is less sensitive to the brightness change. The users may have different requirement on the inference accuracy at different location, and therefore we parameterize the weight by the location context (\texttt{loc}): $w^\texttt{a}_i = w_i(\texttt{loc})$. The variable \texttt{loc} has 10 discrete values, and the value of $w_i(\texttt{loc})$ is sampled from a normal distribution $\mathcal{N}\left(\mu(\texttt{loc}),\delta^2(\texttt{loc})\right)$ where the mean value $\mu(\texttt{loc})$ and the standard deviation $\delta(\texttt{loc})$ are from a uniform distribution $[0,2]$. The inference delay is affected by the CPU temperature ($\texttt{ctemp}\in [0,1]$). Usually, the CPU frequency diminishes at the rate of approximately 150 Hz per degree Celsius, and therefore the inference delay may become larger when the CPU temperature is high. Formally, the inference delay is calculated by $d_i = \texttt{ctemp} \cdot \texttt{ndel}_m$ where $\texttt{ndel}_m$ is the nominal delay of DNN $m$. The weight for the service delay $w^\texttt{d}_i$ depends on the \emph{time} and \emph{battery level}: for example, the user may require faster response during a certain time span; and when the battery level is low, the user may want the computation to be completed in a shorter time to save energy (assuming the CPU frequency is fixed). The time of the delay is discretized into 24 values. For user $i$, we choose a weight $w_i(\texttt{time}) \in \mathcal{N}(1,0.5) $. The state of battery level has two values, $\texttt{battery}\in\{1:\text{`high'}, 2:\text{`low'}\}$. For each user, we set $w_i(\texttt{battery} = 1) = 1$ and choose $w_i(\texttt{battery} = 0) \in \mathcal{N}(3,1)$. The weight for the service delay is determined by $w^\texttt{d}_i = w_i(\texttt{time}) \cdot w_i(\texttt{battery})$.

We simulate 50 users $(i = 1, 2,\dots,50)$, and run our OIC method for each user. The QPN used in NeuralUCB has three fully-connected hidden layers with 8, 16, 8 nodes, respectively. The size of QPN is 4 KB; the inference delay of QPN is $0.50 \pm 0.16ms$; the training delay of QPN is $100.8 \pm 3.5ms$. Therefore, the complexity of NeuralUCB is acceptable to IoT devices. The results shown below give the average performance for all 50 users.

\subsubsection{Performance Comparisons}
To show the superiority of OIC, we compare OIC with four benchmarks. 1) \textbf{Oracle}: Oracle knows user QoE patterns and selects the best-fit DNN in each session. 2) \textbf{LinUCB}: LinUCB \cite{li2010contextual} is a widely-used contextual multi-armed bandit algorithm. It assumes that user QoE is a linear function of context information. 3) \textbf{Fixed-DNN}: one single DNN is used for on-thing inference all the time. 4) \textbf{Random}: a DNN is selected randomly in each application session. 

\begin{figure}[htb]
	\centering
	\subfigure[Average QoE.]
	{
		\includegraphics[width=0.45\linewidth]{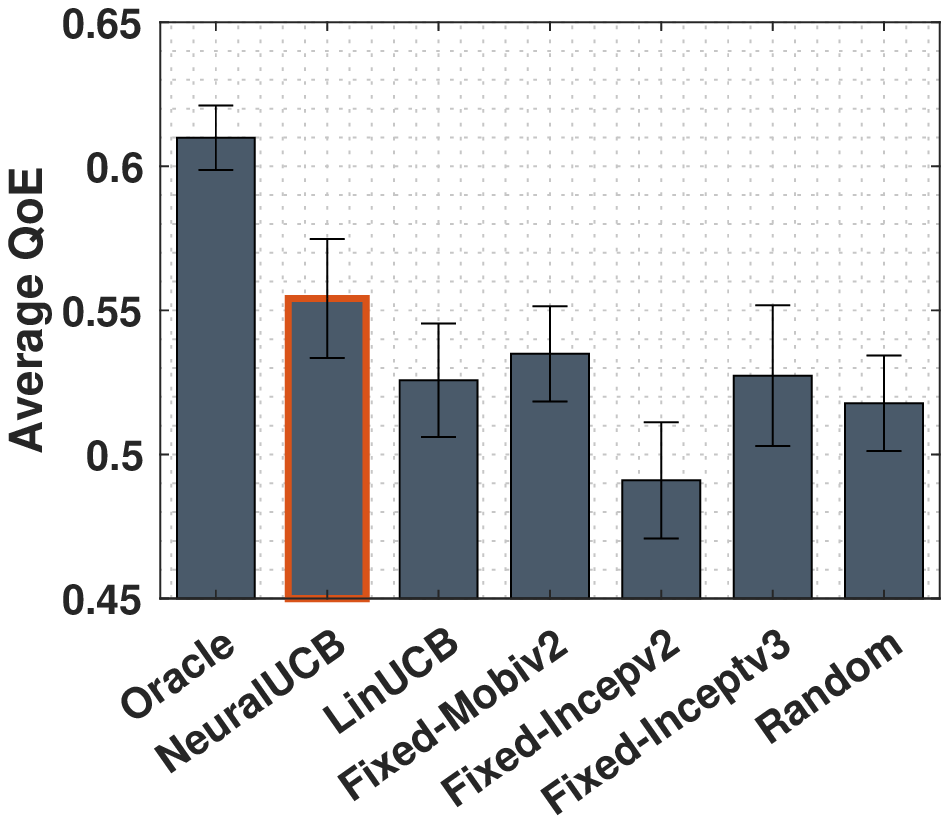}
		\label{fig:ave_qoe_benchmarks}
	}
	\subfigure[Regret.]
	{
		\includegraphics[width=0.45\linewidth]{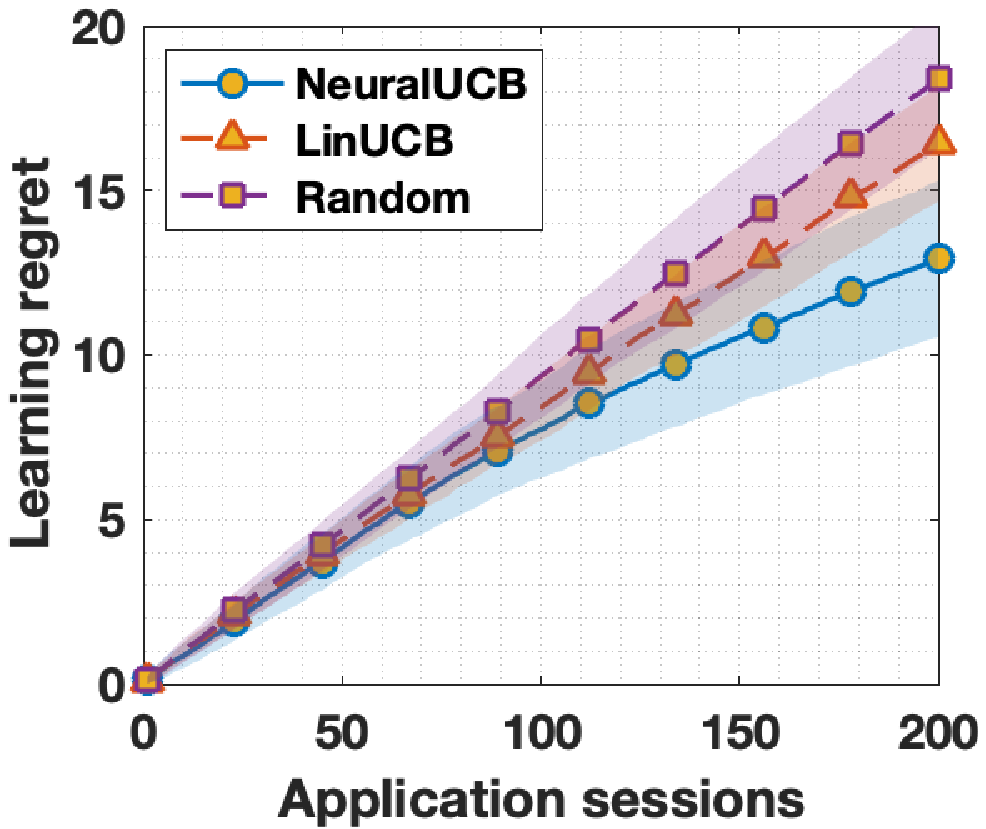}
		\label{fig:regret_benchmarks}
	}
	\caption{Performance comparison with benchmarks.}
	\label{fig:performance_compare}
\end{figure}

Fig. \ref{fig:performance_compare} shows the performance of OIC and other benchmarks in terms of average QoE and regret. Fig. \ref{fig:ave_qoe_benchmarks} 
compares the average QoE (averaged over sessions and users). As expected, Oracle offers the best performance because it knows user QoE patterns and selects the best-fit DNN in each session. Among the others, OIC delivers the highest QoE to users. By comparing OIC with LinUCB, we can infer that on-thing inference customization does not always work if the learning mechanism is inappropriate. Because the constructed QoE function is non-linear, LinUCB is unable to learn user QoE patterns precisely. Selecting DNNs based on inaccurate QoE predictions, LinUCB performs even worse than Fixed-DNN. 

Fig. \ref{fig:regret_benchmarks} shows the regret of NeuralUCB, LinUCB, and Random in 200 sessions. We can see that the regrets of LinUCB and Random grow linearly with the number of sessions. By contrast, the regret of NeuralUCB exhibits a clear sublinear trend, which means that it can achieve asymptotic optimality. 

\begin{figure}[htb]
	\centering
	\subfigure[Average QoE.]
	{
		\includegraphics[width=0.45\linewidth]{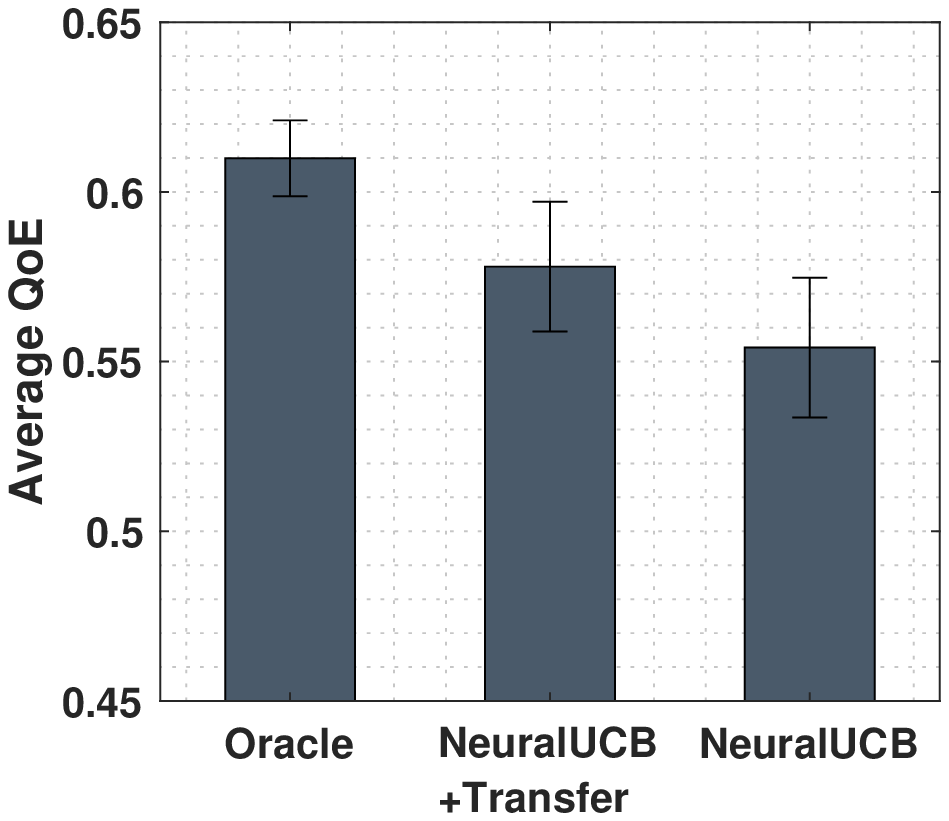}
		\label{fig:ave_qoe_transfer}
	}
	\subfigure[Regret.]
	{    
		\includegraphics[width=0.45\linewidth]{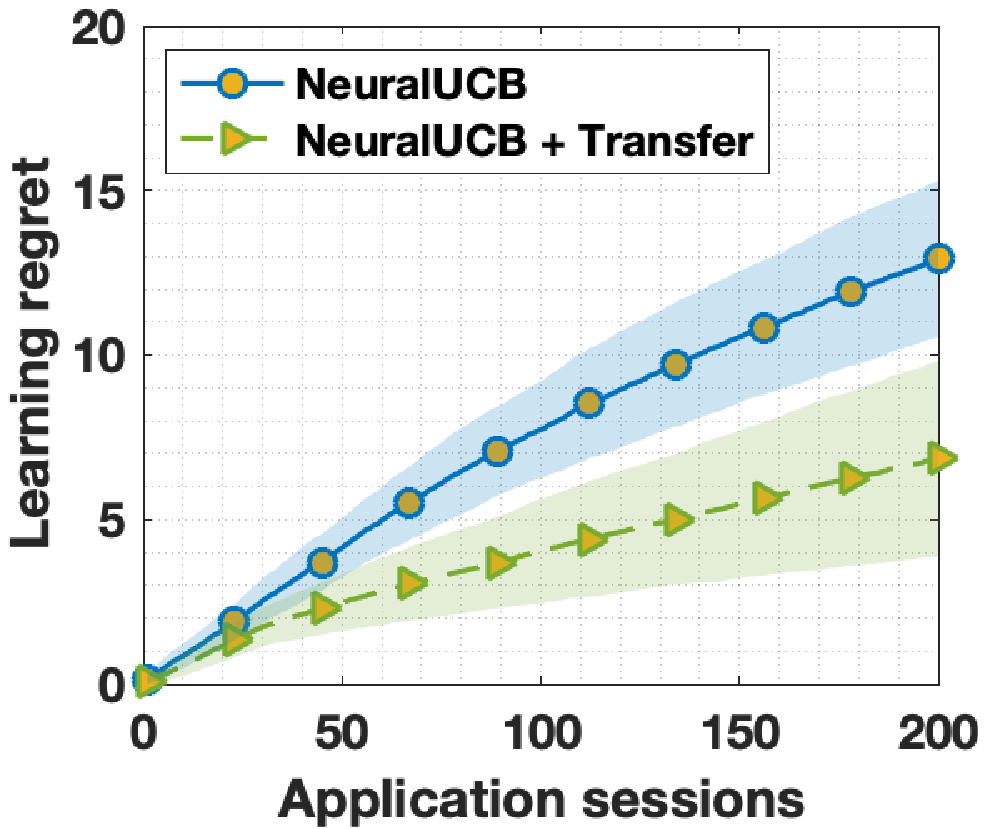}
		\label{fig:regret_transfer}
	}
	\caption{Efficacy of knowledge transfer.}
	\label{fig:knowledge_transfer}
\end{figure}

\subsubsection{Efficacy of Knowledge Transfer}
It can be observed from Fig. \ref{fig:ave_qoe_benchmarks} that the gap between Oracle and NeuralUCB in average QoE is still large. We implement the knowledge transfer technique to reduce this gap. The first step to carry out knowledge transfer is building a pre-trained QPN. We collect 10 QoE samples from each user, and a total of 500 QoE samples are used to pre-train a QPN. Each user is initialized with this pre-trained QPN and then runs NeuralUCB. Fig. \ref{fig:knowledge_transfer} reports the performance of NeuralUCB with knowledge transfer. Fig.\ref{fig:ave_qoe_transfer} shows that incorporating knowledge transfer provides obvious improvements in average QoE. Fig. \ref{fig:regret_transfer} compares the performance in terms of regret. We can see that NeuralUCB+Transfer exhibits a stronger sublinearity in the regret curve, which means that it can find the optimal strategy faster. In the rest of the performance analysis, we run NeuralUCB with knowledge transfer by default.

\begin{figure}[tb]
	\centering
	\includegraphics[width =0.95\linewidth]{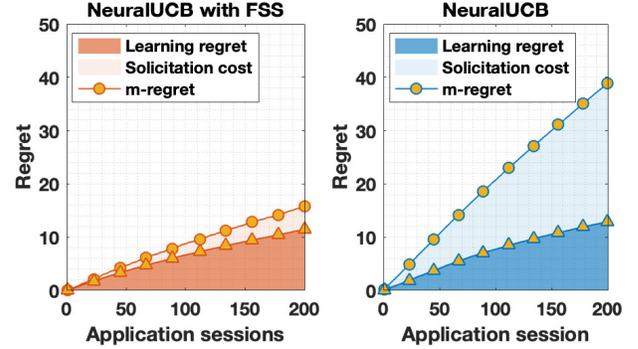}
	\caption{m-Regret of FSS with known time horizon.}
	\vspace{-0.15 in}
	\label{fig:fss}
\end{figure}

\subsubsection{Performance of Feedback Solicitation Schemes}
Fig. \ref{fig:fss} shows m-regret of NeuralUCB when it is applied with and without FSS. We run experiments with $T= 200$ sessions. According to the design, NeuralUCB+FSS only needs $\lceil200^{2/3}\rceil = 35$ solicitations which is much smaller than that of standard NeuralUCB (200 solicitations). Fig. \ref{fig:fss} compares m-regret of NeuralUCB-FSS and NeuralUCB. The unit solicitation cost $\lambda$ is set to 0.13. We see that m-regret of NeuralUCB increases almost linearly due to its high solicitation cost. By contrast, NeuralUCB+FSS achieves a sublinear m-regret. In particular, using FSS does not harm the learning efficiency as we can see that the learning regret of NeuralUCB+FSS is similar to that of NeuralUCB. This means that FSS keeps the asymptotic optimality of NeuralUCB and reduces the increase of solicitation cost to a sublinear rate.

\subsubsection{Performance of FSS-UT}
Fig. \ref{fig:fss_ut_exp} shows the performance of FSS-UT in terms of reward and m-regret. We vary the parameter $\alpha$ in FSS-UT from 0.1 to 1.0. Fig. \ref{fig:alpha_m_regret} provides a decomposition of m-regret after 200 sessions. We see that the solicitation cost decreases with $\alpha$ and the learning regret increase with $\alpha$. The lowest m-regret is achieve at $\alpha = 0.4$ which is close to our theoretical analysis $\alpha = 0.33$. Fig. \ref{fig:alpha_m_regret_curve} shows the m-regret curves for FSS-UT, we see that setting $\alpha \in [0.3,0.9]$ achieves sublinear m-regrets.     
\begin{figure}[tb]
	\centering
	\subfigure[Decomposition of m-regrets.]
	{
		\includegraphics[width=0.45\linewidth]{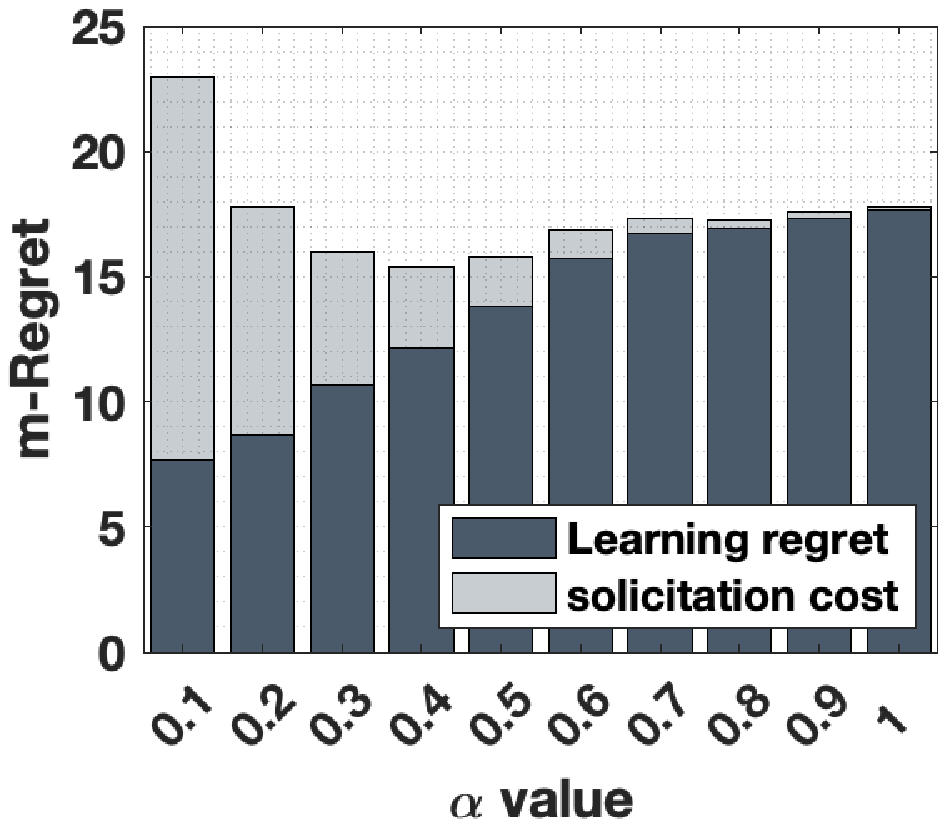}
		\label{fig:alpha_m_regret}
	}
	\subfigure[m-regret curves.]
	{
		\includegraphics[width=0.45\linewidth]{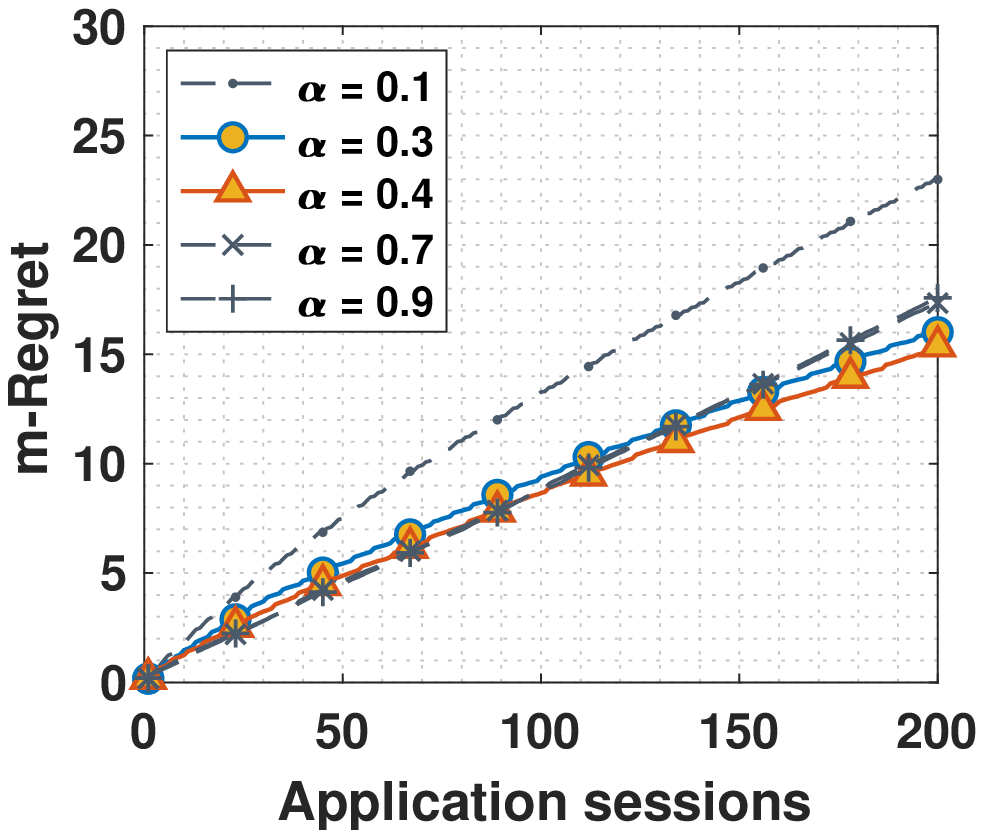}
		\label{fig:alpha_m_regret_curve}
	}
	\caption{Performance of NeuralUCB with FSS-UT.}
	\label{fig:fss_ut_exp}
\end{figure}

\subsubsection{Feedback Refinement for Aggregated QoEs}
We use two schemes to generate aggregated QoE. 

1) The first scheme is \emph{feedback averaging}. It assumes that all used DNNs in a session have the same impact on the aggregated QoE. Feedback averaging generates aggregated QoEs $r^\texttt{AG}_{t}$ using: $r^\texttt{AG}_{t} = 1/K_t \sum_{t = 1}^{K_t} r^k_{t}$ where $r^k_t$ is the QoE for $k$-th DNN used in session $t$. 

2) The second scheme is \emph{sequence-aware aggregation}. It assumes that the impact of DNNs on the aggregated QoE depends on the selection sequence. We let DNNs picked later in the session have a higher impact: $r^\texttt{AG}_{t} = \sum_{k = 1}^K w_k r^k_{t}$ with weights $w_k=2^k/\sum^K_{i=1}2^i$.

\begin{figure}[tb]
	\centering
	\subfigure[Average QoE.]
	{
		\includegraphics[width=0.45\linewidth]{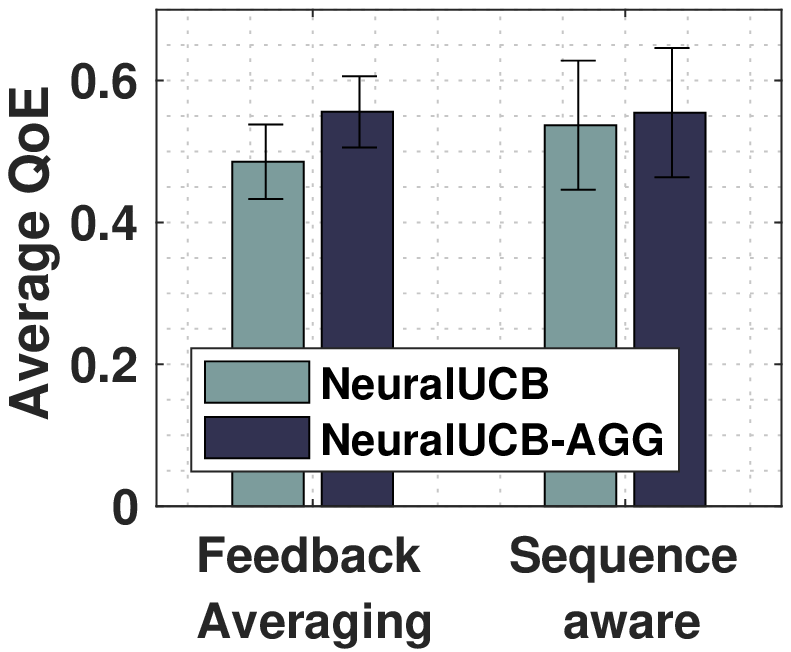}
		\label{fig:fra}
	}
	\subfigure[Mixed QoE feedback.]
	{
		\includegraphics[width=0.45\linewidth]{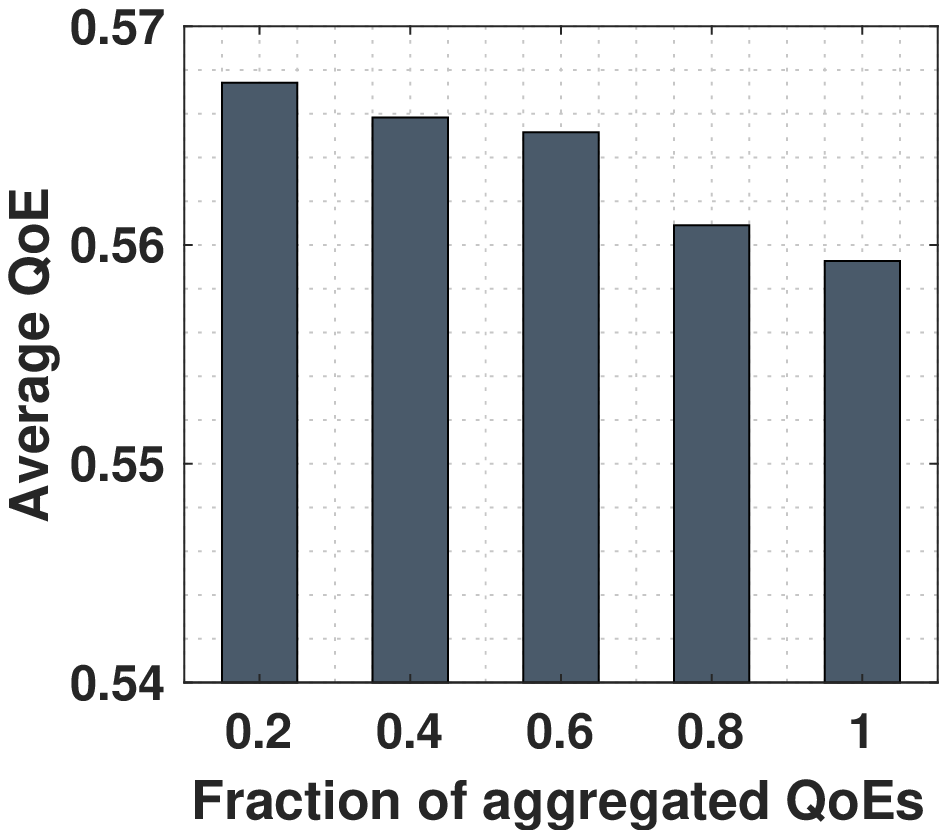}
		\label{fig:mix}
	}
	\caption{Performance of NeuralUCB-AGG.}
	\label{fig:fra_all}
	\vspace{-0.2 in}
\end{figure}

Fig. \ref{fig:fra} compares the performance of NeuralUCB and NeuralUCB-AGG (NeuralUCB with the feedback refinement approach) on aggregated QoEs. Standard NeuralUCB regards the aggregated QoE as the QoE of the last selected DNN in a session. In general, we see that NeuralUCB-AGG delivers higher QoE on aggregated QoEs. We can see that the improvement provided by NeuralUCB-AGG is smaller with sequence-aware aggregation. This is because aggregated QoEs generated by sequence-aware aggregation are dominated by the QoE of the last selected DNN, and hence standard NeuralUCB can still work well. Fig. \ref{fig:mix} shows the performance of NeuralUCB-AGG with mixed feedback where solicited QoEs contain both aggregated QoEs (generated by feedback averaging) and non-aggregated QoEs. We see that NeuralUCB-AGG achieves higher user QoE when the fraction of aggregated QoEs is low. This is because non-aggregated QoEs help estimate individualized QoE in the feedback refinement approach.

\subsection{Experiments on Real-world Dataset}
We also collect context and QoE data from real-world human users to evaluate the proposed method. In this part, we focus on the customization performance and show how our method adapts DNN selection decisions to different user preferences and usage scenarios.

\textbf{User QoE Collection:} We take smartphones as examples of IoT devices. Our data is collected from 5 human subjects on two device: Motorola X$^4$ and OnePlus One. Motorola X$^4$ is equipped with a Qualcomm Snapdragon 625 processor, an octa-core CPU, an Adreno 508 GPU, and a 3GB RAM; OnePlus One is equipped with a Qualcomm Snapdragon 801 processor, a quad-core CPU, an Adreno 330 GPU, and a 3GB RAM. Both smartphones run on the Android 8.1.0 system. We install a DL-based image classification application \cite{tensorflowmodels} on these smartphones. It classifies images captured by the device camera in real-time and displays top-3 classification results. We use a 7-dimension context: the ambient brightness, location of the device, time of the day, battery level, CPU temperature, nominal accuracy of DNNs, nominal delay of DNNs. At the beginning of each session, the subjects record the context. Then a DNN (MobileNet-v2, Inception-v2, or Inception-v3) is randomly picked to configure the application. The subjects are requested to use the application to classify objects in the surrounding environment during the session (around 2 minutes). At the end of the session, subjects provide their QoE ratings. The context, selection DNN, and QoE rating are stored in a dataset. We collect 200 data samples for each subject. However, we cannot directly run our method on the collected data due to the lack of counterfactual QoE data. Specifically, each collected data sample records the context information and QoE for a used DNN. When we feed the context of a data sample into OIC, the DNN selected by our method can be different from the DNN we used for collecting the user QoE. As a result, we cannot obtain the ground-truth user QoE for the selected DNN. In this experiment, we utilize the counterfactual prediction method in \cite{hartford2017deep} to generate the counterfactual QoE data.

\begin{figure*}[htb]
	\centering
	\includegraphics[width = 1\linewidth]{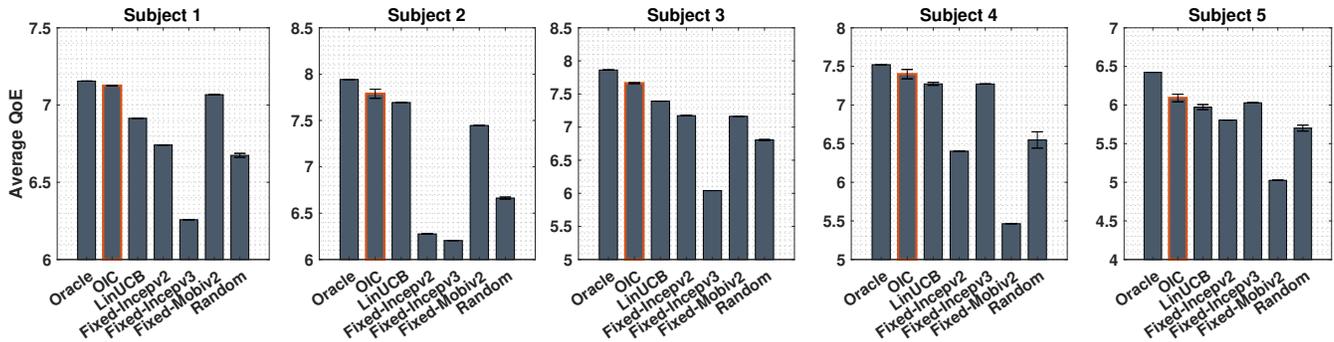}
	\captionof{figure}{Average QoE of subjects.}
	\label{fig:ave_qoe_subject}
\end{figure*}

\subsubsection{Comparison on Average QoE}
Fig. \ref{fig:ave_qoe_subject} shows (session-)average QoE of 5 subjects when using OIC and other benchmarks. Similar to that in the numerical experiment, OIC outperforms other benchmarks and achieves a close-to-oracle performance. In some cases, the performance of Fixed-DNN is comparable to OIC, e.g., Fixed-Mobiv2 for user 1 and Fix-Incepv3 for user 4. However, we can see that fixing a certain DNN cannot deliver high QoE to all subjects and hence on-thing inference customization is very necessary. 

\begin{figure*}[htb]
	\centering
	\includegraphics[width = 1\linewidth]{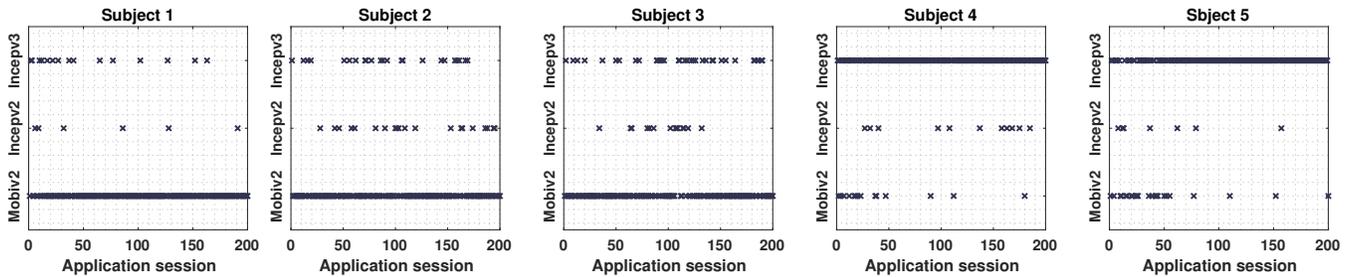}
	\captionof{figure}{Selected DNNs for subjects.}
	\label{fig:user_action}
\end{figure*}
\begin{figure*}[htb]
	\centering
	\includegraphics[width=1\linewidth]{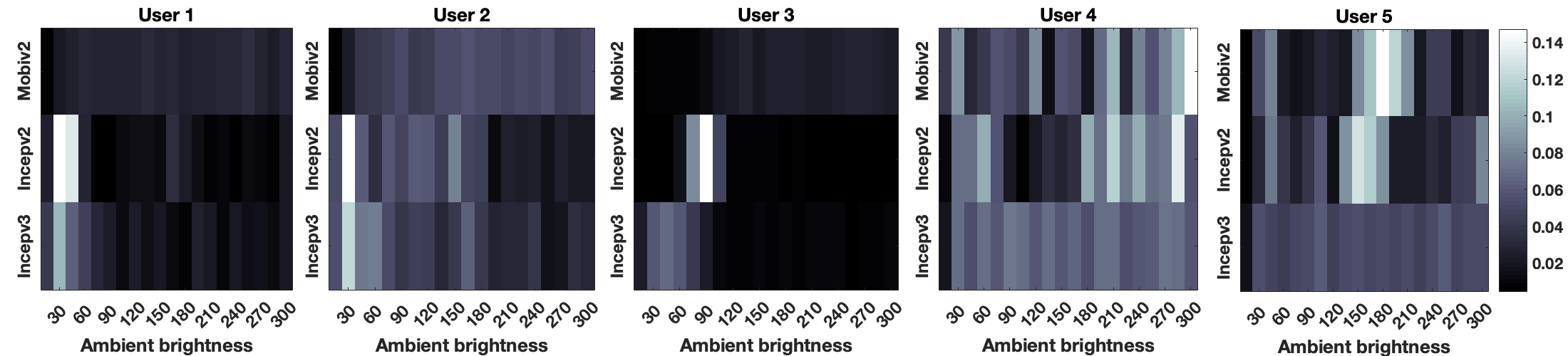}
	\caption{Impact of ambient brightness on DNN selection.}
	\label{fig:ambient_context}
\end{figure*}

\subsubsection{Heterogeneity of User Preference}
Fig. \ref{fig:user_action} shows selected DNNs for 5 subjects in 200 sessions. For ease of interpretation, we give the nominal performance of three DNNs: Mobiv2 (accuracy 70.8\%, inference delay 12ms), Incepv2 (accuracy 73.5\%, inference delay 59ms), Incepv3 (accuracy 77.5\%, inference delay 148ms). It can be observed that Mobiv2 is often selected for the subject 1, 2, and 3, meaning that these users prefer low inference delay, and Incepv3 is often selected for the subject 4 and 5, meaning that these users prefer high inference accuracy. We see that OIC is able to adapt its DNN selection scheme to different user preferences.

\subsubsection{Adaption to Environment Changes}
We next take ambient brightness as an example to show the impact of context on the DNN selection. Fig. \ref{fig:ambient_context} shows the fraction that a DNN is selected under different ambient brightness. For user 1, 2, and 3 (prefer low latency), we can see that when the ambient brightness is relatively low (below 100), Incepv2 and Incepv3 are more likely to be selected. This is because low brightness causes accuracy degradation, and in this case, the users may trade low inference delay (i.e., selecting Mobiv2) for better inference accuracy (i.e., selecting Incepv2 and Incepv3). For user 5, we see that Mobiv2 and Incepv2 are more likely to be selected when the brightness is appropriate (between 120-200). This is because all three DNNs can achieve high accuracy when the brightness is appropriate, and hence DNNs with low inference delay are preferred. These results indicate that OIC can adapt to environmental changes.

\section{Conclusion} \label{sec:conclusion}
In this paper, we designed automated on-thing inference customization for DL-based applications. Our goal is to improve the user-perceived experience by selecting the best-fit DNN for configuring the application. Two main designing topics, QoE prediction and feedback solicitation scheme, are investigated. A novel multi-armed bandit algorithm called NeuralUCB was utilized to learn a QoE predictor. NeuralUCB exhibits excellent generalization ability for handling heterogeneous user QoE patterns. We also showed via experiments that the space and time complexity of NeuralUCB is acceptable to resource-constrained devices. The feedback solicitation scheme (FSS) was further studied to mitigate the solicitation cost during online learning. With FSS, we dramatically reduced the number of QoE solicitations required by NeuralUCB without harming its asymptotic optimality. The proposed framework can be applied to many other multi-armed bandit learning problems where online feedback collections cause non-negligible cost.          

\bibliographystyle{IEEEtran}
\bibliography{references}

\vskip 0pt plus -1fil
\begin{IEEEbiography}[{\includegraphics[width=1in,height=1.25in,clip,keepaspectratio]{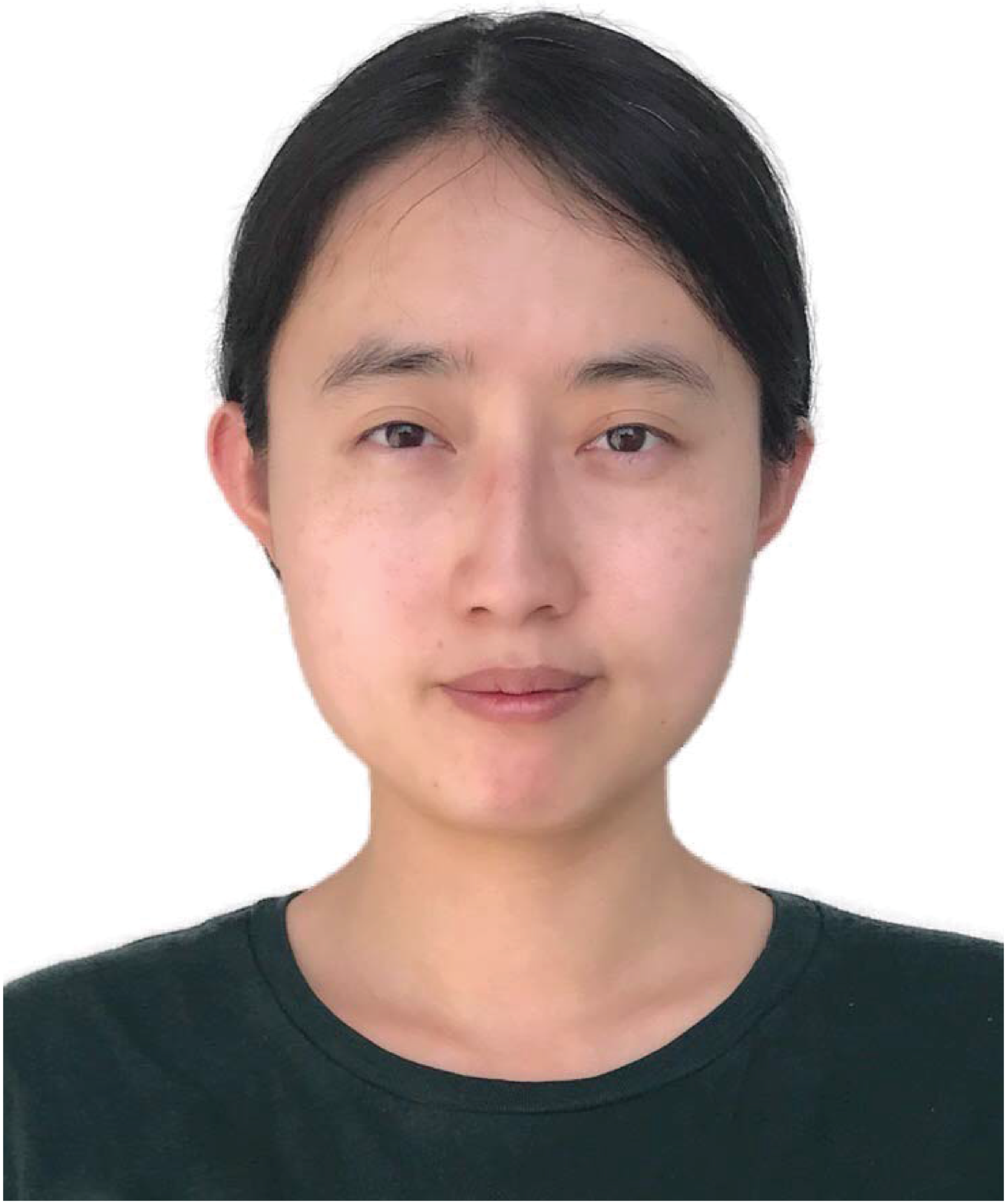}}]{Yang Bai} received her Ph.D. degree in the College of Engineering, University of Miami, USA, in 2021. She received the B.S. degree from Northeastern University, Shenyang, China, and the M.S. degree from the College of Engineering, University of Miami. Her primary research interests include machine learning, computer vision, and game theory.
\end{IEEEbiography}

\vskip 0pt plus -1fil

\begin{IEEEbiography}[{\includegraphics[width=1in,height=1.25in,clip,keepaspectratio]{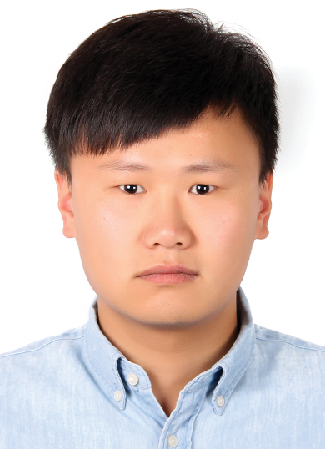}}]{Lixing Chen} is a tenure-track assistant professor in Institute of Cyber Science and Technology at Shanghai Jiao Tong University, China. He received the Ph.D. degree in Electrical and Computer Engineering from the University of Miami in 2020, and the BS and ME Degrees from the College of Information and Control Engineering, China University of Petroleum, Qingdao, China, in 2013 and 2016, respectively. His primary research interests include mobile edge computing and machine learning for networks. 
\end{IEEEbiography}

\vskip 0pt plus -1fil

\begin{IEEEbiography}[{\includegraphics[width=1in,height=1.25in,clip,keepaspectratio]{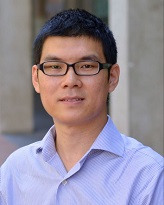}}]{Shaolei Ren} (Senior Member, IEEE) received the BE degree from Tsinghua University, in 2006, the MPhil degree from Hong Kong University of Science and Technology, in 2008, and the PhD degree from the University of California, Los Angeles, in 2012, all in electrical and computer engineering. He is an assistant professor of electrical and computer engineering with the University of California, Riverside. His research interests include cloud computing, data centers, and network economics. He was a recipient of the U.S. NSF Faculty Early Career Development (CAREER) Award in 2015.
	
\end{IEEEbiography}

\vskip 0pt plus -1fil

\begin{IEEEbiography}[{\includegraphics[width=1in,height=1.25in,clip,keepaspectratio]{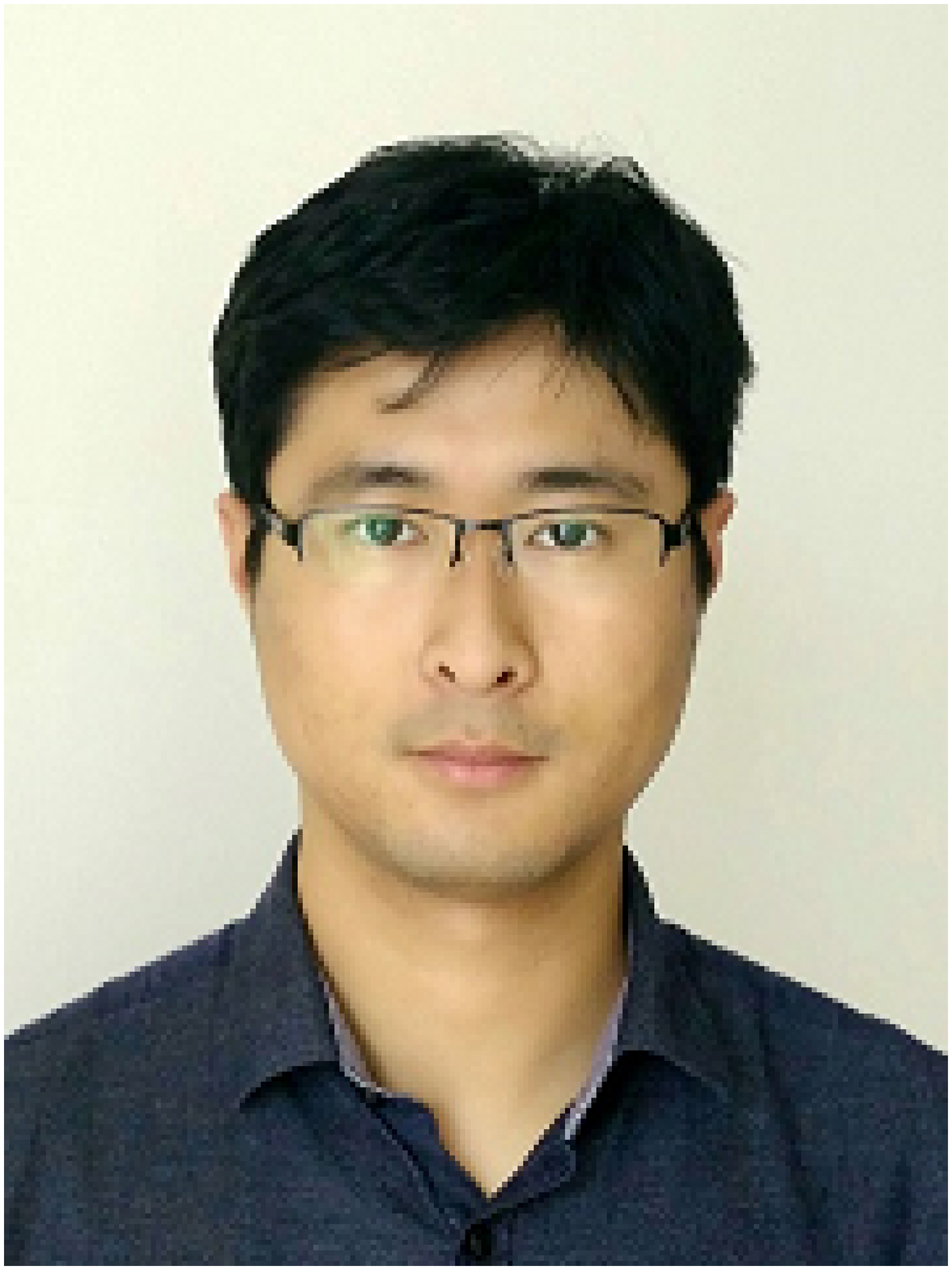}}]{Jie Xu} (Senior Member, IEEE) received the B.S. and M.S. degrees in electronic engineering from Tsinghua University, Beijing, China, in 2008 and 2010, respectively, and the Ph.D. degree in electrical engineering from UCLA in 2015. He is currently an Associate Professor with the Department of Electrical and Computer Engineering, University of Miami. His research interests include mobile edge computing/intelligence, machine learning for networks, and network security. He received the NSF CAREER Award in 2021. 
\end{IEEEbiography}

\clearpage
\appendices
\section{Proof of Theorem \ref{theo:bound_m_regret}} \label{proof:theo:bound_m_regret}
\begin{proof}
	Using FSS, a QoE feedback is solicited every $B = \lfloor T^{1/3}\rfloor$ application sessions. Let $\mathcal{L} = \{1,2,\dots,L\}$, where $L = \min(\{l\in\mathbb{Z}^+~|~l \cdot B > T\})$, be the index of solicitations. $l$-th QoE solicitation happens in session $t = (l-1)B + 1$. Therefore, we can rewrite the m-regret as: 
	\begin{equation}\label{eq:regret_slotted}
	\begin{split}
	\tilde{R}_T \leq & \mathbb{E}\left[\sum\nolimits_{l=1}^L \sum\nolimits_{t = (l-1)B + 1}^{lB}\tilde{r}_{t,m_t^*}-\tilde{r}_{t,m_t}\right]\\
	= &\mathbb{E}\left[\sum\nolimits_{l=1}^L \sum\nolimits_{t = (l-1)B+1}^{lB}r_{t,m_t^*}-r_{t,m_t}\right] + \lambda|\mathcal{L}|
	\end{split}
	\end{equation}
	
	The term $\mathcal{L}$ in the last equation is the solicitation cost and its value equals $\lambda L$. To bound the learning regret of NeuralUCB with the designed feedback solicitation scheme, we state a simplified version of Lemma 5.3 \cite{zhou2019neural}.
	\begin{lemma}\label{lemma:neuralUCB_perslot}
		For a DNN $m_t$ selected by NeuralUCB in session $t$, with probability $1-\delta$, we have 
		$r_{t,m^*_t} - r_{t,m_t} \leq \mathcal{Q}(x_{t,m_t},l,\delta)$, where $\mathcal{Q}(x_{t,m_t},l,\delta) :=\gamma_{l} \min\left\{\|g(x_{t,m_t};\theta_{l})/\sqrt{m(\delta)}\|_{\bm{Z}^{-1}_{l}},1\right\} + \mathcal{G}(T,\delta)$. The parameter $l$ is the amount of collected QoE data up to session $t$, $m(\delta)$ is a constant determined by $\delta$, and $\mathcal{G}(T,\delta)$ is a function of $T$ and $\delta$.
	\end{lemma}
	Letting $t_l = \argmax\limits_{l\leq t \leq l+ B -1} \mathcal{Q}(x_{t,m_t},l,\delta)$ and using Lemma \ref{lemma:neuralUCB_perslot}, we can write the learning regret in \eqref{eq:regret_slotted} as:
	\begin{align*}
	&\sum\nolimits_{l=1}^{L} \sum\nolimits_{t =  (l-1)B+1}^{lB} r_{t,m_t^*}-r_{t,m_t}
	\\ \leq & \sum\nolimits_{l=1}^L \sum\nolimits_{t = (l-1)B+1}^{lB} \mathcal{Q}(x_{t,m_t},l,\delta) \\
	\leq & B\sum\nolimits_{l=1}^L \mathcal{Q}(x_{t_l,m_{t_l}},l,\delta)
	\end{align*}
	Consider the upper regret bound $\mathcal{O}(\sqrt{T})$ of standard NeuralUCB in Lemma \ref{lemma:neuralucb_regret_bound} where the QoE is solicited every session $t$, we have  
	\begin{align*}
	\sum\nolimits_{t=1} ^T r_{t,m^*_t} - r_{t,m_t} \leq \sum\nolimits^T_{t=1} \mathcal{Q}(x_{t,m_t},l,\delta)
	=  \mathcal{O}(\sqrt{T})
	\end{align*}
	Similarly, we will have
	\begin{align*}
	B\sum\nolimits_{l=1}^L \mathcal{Q}(x_{t_l,m_{t_l}},l,\delta)
	= \mathcal{O}(B\sqrt{L}).
	\end{align*}
	Then, the m-regret of NeuralUCB with FSS is 
	\begin{align*}
	\tilde{R}_T = & \mathcal{O}(B\sqrt{L}) + L \\
	= & \mathcal{O}\left(\lfloor T^{1/3} \rfloor\sqrt{\left\lceil T/T^{1/3}\right\rceil} \right)  + \lambda \left\lceil T/T^{1/3} \right\rceil
	\end{align*}
	Using $\lfloor T^{1/3} \rfloor\sqrt{\lceil T/T^{1/3} \rceil} \leq T^{1/3} \sqrt{T/T^{1/3} + 1}\leq T^{1/3}(\sqrt{T^{2/3}}+1)$ completes the proof.
\end{proof}

\section{Proof of Theorem \ref{theo:regret_fss_ut}}\label{proof:theo:regret_fss_ut}
We let $\mathcal{L} = \{1,2,\dots\}$ be the index of solicitations. According to the design of FSS-UT, we know that the $l$-th solicitation happens in session $t=\lceil l^{1/(1-\alpha)} \rceil$. Suppose our algorithm runs $T$ sessions, then the total number of solicitation $L$ is $\lceil T^{1-\alpha} \rceil$. Using Lemma \ref{lemma:neuralUCB_perslot}, we can write the learning regret of NeuralUCB with FSS-UT as 
\begin{align*}
& \sum\nolimits_{l=1}^{L} \sum\nolimits_{t = \lceil l^{1/(1-\alpha)}\rceil}^{ \lceil (l+1)^{1/(1-\alpha)}\rceil -1} r_{t,m_t^*}-r_{t,m_t}\\
\leq & \sum\nolimits_{l=1}^L \sum\nolimits_{t = \lceil l^{1/(1-\alpha)}\rceil}^{ \lceil (l+1)^{1/(1-\alpha)}\rceil -1} \mathcal{Q}(x_{t,m_t},l,\delta) \\
\leq & \sum\nolimits_{l=1}^L \left(\lceil (l+1)^{1/(1-\alpha)}\rceil -\lceil l^{1/(1-\alpha)}\rceil\right)\cdot \mathcal{Q}(x_{t_l,m_{t_l}},l,\delta)\\
\leq & \sum\nolimits_{l=1}^L \left(\lceil (L+1)^{1/(1-\alpha)}\rceil -\lceil L^{1/(1-\alpha)}\rceil\right)\cdot \mathcal{Q}(x_{t_l,m_{t_l}},l,\delta)
\end{align*}
Note that $\mathcal{O}(\lceil (L+1)^{1/(1-\alpha)}\rceil -\lceil L^{1/(1-\alpha)}\rceil\ ) = \mathcal{O}(L^{1/(1-\alpha)-1}) = \mathcal{O}(T^\alpha)$. The order of learning regret is $\mathcal{O}(T^{(\alpha+1)/2})$, and the order of solicitation cost is $\mathcal{O}(T^{1-\alpha})$


%

\ifCLASSOPTIONcaptionsoff
  \newpage
\fi

\end{document}

%% file: macros.tex
\DeclareMathOperator*{\argmax}{arg\,max} 

\newcommand{\bp}{\begin{proof} \small }
\newcommand{\ep}{\end{proof} \normalsize}
\newcommand{\epx}{\end{proof} \small}
\newcommand{\bpa}{\begin{proofappx} \footnotesize }
\newcommand{\epa}{\end{proofappx} \small }
\newtheorem{theorem}{Theorem}
\newtheorem{proposition}{Proposition}

\newtheorem{lemma}{Lemma}

\newtheorem*{theorem*}{Theorem}
\newtheorem*{proposition*}{Proposition}
\newtheorem*{corollary*}{Corollary}
\newtheorem*{lemma*}{Lemma}
\newtheorem*{assumption*}{Assumption}
\newtheorem*{definition*}{Definition}
\newtheorem*{claim*}{Claim}

\newcommand{\bm}[1]{\mbox{\boldmath $#1$}}

\newcommand{\be}{\begin{equation}}
\newcommand{\ee}{\end{equation}}
\newcommand{\bs}{\begin{subequations}}
\newcommand{\es}{\end{subequations}}
\newcommand{\bq}{\begin{eqnarray}}
\newcommand{\eq}{\end{eqnarray}}
\newcommand{\bqn}{\begin{eqnarray*}}
\newcommand{\eqn}{\end{eqnarray*}}

\newcommand{\ba}{\left[ \begin{array}}
\newcommand{\ea}{\\ \end{array} \right]}
\newcommand{\ben}{\begin{enumerate}}
\newcommand{\een}{\end{enumerate}}

\def\real{{\mathchoice%
{\hbox{\rm\setbox1=\hbox{I}\copy1\kern-.45\wd1 R}}
{\hbox{\rm\setbox1=\hbox{I}\copy1\kern-.45\wd1 R}}
{\hbox{\scriptsize\rm\setbox1=\hbox{I}\copy1\kern-.45\wd1 R}}
{\hbox{\scriptsize\rm\setbox1=\hbox{I}\copy1\kern-.45\wd1 R}}}}

\def\Zint{{\mathchoice{\setbox1=\hbox{\sf Z}\copy1\kern-.75\wd1\box1}
{\setbox1=\hbox{\sf Z}\copy1\kern-.75\wd1\box1}
{\setbox1=\hbox{\scriptsize\sf Z}\copy1\kern-.75\wd1\box1}
{\setbox1=\hbox{\scriptsize\sf Z}\copy1\kern-.75\wd1\box1}}}
\newcommand{\complex}{ \hbox{\rm C\kern-0.45em\rule[.07em]{.02em}{.58em}%
\kern 0.43em}}

\makeatletter
\newcommand{\algmargin}{\the\ALG@thistlm}
\makeatother
\newlength{\whilewidth}
\algdef{SE}[parWHILE]{parWhile}{EndparWhile}[1]
{\parbox[t]{\dimexpr\linewidth-\algmargin}{%
		\hangindent\whilewidth\strut\algorithmicwhile\ #1\ \algorithmicdo\strut}}{\algorithmicend\ \algorithmicwhile}%
\algnewcommand{\parState}[1]{\State%
	\parbox[t]{\dimexpr\linewidth-\algmargin}{\strut #1\strut}}

\ifodd 1

\else

\fi

%% file: QoE_main.bbl
\begin{thebibliography}{10}
\providecommand{\url}[1]{#1}
\csname url@samestyle\endcsname
\providecommand{\newblock}{\relax}
\providecommand{\bibinfo}[2]{#2}
\providecommand{\BIBentrySTDinterwordspacing}{\spaceskip=0pt\relax}
\providecommand{\BIBentryALTinterwordstretchfactor}{4}
\providecommand{\BIBentryALTinterwordspacing}{\spaceskip=\fontdimen2\font plus
\BIBentryALTinterwordstretchfactor\fontdimen3\font minus
  \fontdimen4\font\relax}
\providecommand{\BIBforeignlanguage}[2]{{%
\expandafter\ifx\csname l@#1\endcsname\relax
\typeout{** WARNING: IEEEtran.bst: No hyphenation pattern has been}%
\typeout{** loaded for the language `#1'. Using the pattern for}%
\typeout{** the default language instead.}%
\else
\language=\csname l@#1\endcsname
\fi
#2}}
\providecommand{\BIBdecl}{\relax}
\BIBdecl

\bibitem{zhang2019deep}
C.~Zhang, P.~Patras, and H.~Haddadi, ``Deep learning in mobile and wireless
  networking: A survey,'' \emph{IEEE Communications Surveys \& Tutorials},
  vol.~21, no.~3, pp. 2224--2287, 2019.

\bibitem{deng2014deep}
L.~Deng and D.~Yu, ``Deep learning: methods and applications,''
  \emph{Foundations and trends in signal processing}, vol.~7, no. 3--4, pp.
  197--387, 2014.

\bibitem{eshratifar2019jointdnn}
A.~E. Eshratifar, M.~S. Abrishami, and M.~Pedram, ``Jointdnn: an efficient
  training and inference engine for intelligent mobile cloud computing
  services,'' \emph{IEEE Transactions on Mobile Computing}, 2019.

\bibitem{li2017multi}
P.~Li, J.~Li, Z.~Huang, T.~Li, C.-Z. Gao, S.-M. Yiu, and K.~Chen, ``Multi-key
  privacy-preserving deep learning in cloud computing,'' \emph{Future
  Generation Computer Systems}, vol.~74, pp. 76--85, 2017.

\bibitem{ran2018deepdecision}
X.~Ran, H.~Chen, X.~Zhu, Z.~Liu, and J.~Chen, ``Deepdecision: A mobile deep
  learning framework for edge video analytics,'' in \emph{IEEE INFOCOM
  2018-IEEE Conference on Computer Communications}.\hskip 1em plus 0.5em minus
  0.4em\relax IEEE, 2018, pp. 1421--1429.

\bibitem{freire2019deep}
D.~Freire-Obreg{\'o}n, F.~Narducci, S.~Barra, and M.~Castrill{\'o}n-Santana,
  ``Deep learning for source camera identification on mobile devices,''
  \emph{Pattern Recognition Letters}, vol. 126, pp. 86--91, 2019.

\bibitem{lane2015can}
N.~D. Lane and P.~Georgiev, ``Can deep learning revolutionize mobile sensing?''
  in \emph{Proceedings of the 16th International Workshop on Mobile Computing
  Systems and Applications}, 2015, pp. 117--122.

\bibitem{wu2019machine}
C.-J. Wu, D.~Brooks, K.~Chen, D.~Chen, S.~Choudhury, M.~Dukhan, K.~Hazelwood,
  E.~Isaac, Y.~Jia, B.~Jia \emph{et~al.}, ``Machine learning at facebook:
  Understanding inference at the edge,'' in \emph{2019 IEEE International
  Symposium on High Performance Computer Architecture (HPCA)}.\hskip 1em plus
  0.5em minus 0.4em\relax IEEE, 2019, pp. 331--344.

\bibitem{mao2017mobile}
Y.~Mao, C.~You, J.~Zhang, K.~Huang, and K.~B. Letaief, ``Mobile edge computing:
  Survey and research outlook,'' \emph{arXiv preprint arXiv:1701.01090}, 2017.

\bibitem{sima2018apple}
D.~Sima, ``Apple’s mobile processors,'' 2018.

\bibitem{lite2017android}
T.~Lite, ``Android to launch tensorflow lite for mobile machine learning,''
  2017.

\bibitem{coreML}
\BIBentryALTinterwordspacing
A.~Inc. Core ml: Integrate machine learning models into your app. [Online].
  Available: \url{https://developer.apple.com/documentation/coreml}
\BIBentrySTDinterwordspacing

\bibitem{xie2019source}
X.~Xie and K.-H. Kim, ``Source compression with bounded dnn perception loss for
  iot edge computer vision,'' in \emph{The 25th Annual International Conference
  on Mobile Computing and Networking}, 2019, pp. 1--16.

\bibitem{zhang2018systematic}
T.~Zhang, S.~Ye, K.~Zhang, J.~Tang, W.~Wen, M.~Fardad, and Y.~Wang, ``A
  systematic dnn weight pruning framework using alternating direction method of
  multipliers,'' in \emph{Proceedings of the European Conference on Computer
  Vision (ECCV)}, 2018, pp. 184--199.

\bibitem{wang2019private}
J.~Wang, W.~Bao, L.~Sun, X.~Zhu, B.~Cao, and S.~Y. Philip, ``Private model
  compression via knowledge distillation,'' in \emph{Proceedings of the AAAI
  Conference on Artificial Intelligence}, vol.~33, 2019, pp. 1190--1197.

\bibitem{taylor2018adaptive}
B.~Taylor, V.~S. Marco, W.~Wolff, Y.~Elkhatib, and Z.~Wang, ``Adaptive deep
  learning model selection on embedded systems,'' \emph{ACM SIGPLAN Notices},
  vol.~53, no.~6, pp. 31--43, 2018.

\bibitem{park2015big}
E.~Park, D.~Kim, S.~Kim, Y.-D. Kim, G.~Kim, S.~Yoon, and S.~Yoo, ``Big/little
  deep neural network for ultra low power inference,'' in \emph{Proceedings of
  the 10th International Conference on Hardware/Software Codesign and System
  Synthesis}.\hskip 1em plus 0.5em minus 0.4em\relax IEEE Press, 2015, pp.
  124--132.

\bibitem{cai2017neuralpower}
E.~Cai, D.-C. Juan, D.~Stamoulis, and D.~Marculescu, ``Neuralpower: Predict and
  deploy energy-efficient convolutional neural networks,'' \emph{arXiv preprint
  arXiv:1710.05420}, 2017.

\bibitem{rouhani2016delight}
B.~D. Rouhani, A.~Mirhoseini, and F.~Koushanfar, ``Delight: Adding energy
  dimension to deep neural networks,'' in \emph{Proceedings of the 2016
  International Symposium on Low Power Electronics and Design}, 2016, pp.
  112--117.

\bibitem{chen2014qos}
Y.~Chen, K.~Wu, and Q.~Zhang, ``From qos to qoe: A tutorial on video quality
  assessment,'' \emph{IEEE Communications Surveys \& Tutorials}, vol.~17,
  no.~2, pp. 1126--1165, 2014.

\bibitem{barakovic2013survey}
S.~Barakovi{\'c} and L.~Skorin-Kapov, ``Survey and challenges of qoe management
  issues in wireless networks,'' \emph{Journal of Computer Networks and
  Communications}, vol. 2013, 2013.

\bibitem{edgetpu}
\BIBentryALTinterwordspacing
G.~Cloud. (2019) Edge tpu - run inference at the edge. [Online]. Available:
  \url{https://cloud.google.com/edge-tpu}
\BIBentrySTDinterwordspacing

\bibitem{howard2017mobilenets}
A.~G. Howard, M.~Zhu, B.~Chen, D.~Kalenichenko, W.~Wang, T.~Weyand,
  M.~Andreetto, and H.~Adam, ``Mobilenets: Efficient convolutional neural
  networks for mobile vision applications,'' \emph{arXiv preprint
  arXiv:1704.04861}, 2017.

\bibitem{szegedy2017inception}
C.~Szegedy, S.~Ioffe, V.~Vanhoucke, and A.~A. Alemi, ``Inception-v4,
  inception-resnet and the impact of residual connections on learning,'' in
  \emph{Thirty-first AAAI conference on artificial intelligence}, 2017.

\bibitem{zoph2018learning}
B.~Zoph, V.~Vasudevan, J.~Shlens, and Q.~V. Le, ``Learning transferable
  architectures for scalable image recognition,'' in \emph{Proceedings of the
  IEEE conference on computer vision and pattern recognition}, 2018, pp.
  8697--8710.

\bibitem{redmon2016you}
J.~Redmon, S.~Divvala, R.~Girshick, and A.~Farhadi, ``You only look once:
  Unified, real-time object detection,'' in \emph{Proceedings of the IEEE
  conference on computer vision and pattern recognition}, 2016, pp. 779--788.

\bibitem{cheng2018model}
Y.~Cheng, D.~Wang, P.~Zhou, and T.~Zhang, ``Model compression and acceleration
  for deep neural networks: The principles, progress, and challenges,''
  \emph{IEEE Signal Processing Magazine}, vol.~35, no.~1, pp. 126--136, 2018.

\bibitem{han2015deep}
S.~Han, H.~Mao, and W.~J. Dally, ``Deep compression: Compressing deep neural
  networks with pruning, trained quantization and huffman coding,'' \emph{arXiv
  preprint arXiv:1510.00149}, 2015.

\bibitem{huang2018data}
Z.~Huang and N.~Wang, ``Data-driven sparse structure selection for deep neural
  networks,'' in \emph{Proceedings of the European conference on computer
  vision (ECCV)}, 2018, pp. 304--320.

\bibitem{tensorflowmodels}
\BIBentryALTinterwordspacing
T.~Lite. (2020) Hosted dnn models. [Online]. Available:
  \url{https://www.tensorflow.org}
\BIBentrySTDinterwordspacing

\bibitem{windrim2016unsupervised}
L.~Windrim, A.~Melkumyan, R.~Murphy, A.~Chlingaryan, and J.~Nieto,
  ``Unsupervised feature learning for illumination robustness,'' in \emph{2016
  IEEE International Conference on Image Processing (ICIP)}.\hskip 1em plus
  0.5em minus 0.4em\relax IEEE, 2016, pp. 4453--4457.

\bibitem{carson2007incentive}
R.~T. Carson and T.~Groves, ``Incentive and informational properties of
  preference questions,'' \emph{Environmental and resource economics}, vol.~37,
  no.~1, pp. 181--210, 2007.

\bibitem{zhou2019neural}
D.~Zhou, L.~Li, and Q.~Gu, ``Neural contextual bandits with upper confidence
  bound-based exploration,'' \emph{arXiv preprint arXiv:1911.04462}, 2019.

\bibitem{han2017ese}
S.~Han, J.~Kang, H.~Mao, Y.~Hu, X.~Li, Y.~Li, D.~Xie, H.~Luo, S.~Yao, Y.~Wang
  \emph{et~al.}, ``Ese: Efficient speech recognition engine with sparse lstm on
  fpga,'' in \emph{Proceedings of the 2017 ACM/SIGDA International Symposium on
  Field-Programmable Gate Arrays}, 2017, pp. 75--84.

\bibitem{stamoulis2018designing}
D.~Stamoulis, T.-W.~R. Chin, A.~K. Prakash, H.~Fang, S.~Sajja, M.~Bognar, and
  D.~Marculescu, ``Designing adaptive neural networks for energy-constrained
  image classification,'' in \emph{Proceedings of the International Conference
  on Computer-Aided Design}.\hskip 1em plus 0.5em minus 0.4em\relax ACM, 2018,
  p.~23.

\bibitem{lu2019automating}
B.~Lu, J.~Yang, L.~Y. Chen, and S.~Ren, ``Automating deep neural network model
  selection for edge inference,'' in \emph{2019 IEEE First International
  Conference on Cognitive Machine Intelligence (CogMI)}.\hskip 1em plus 0.5em
  minus 0.4em\relax IEEE, 2019, pp. 184--193.

\bibitem{langford2007epoch}
J.~Langford and T.~Zhang, ``The epoch-greedy algorithm for contextual
  multi-armed bandits,'' in \emph{Proceedings of the 20th International
  Conference on Neural Information Processing Systems}.\hskip 1em plus 0.5em
  minus 0.4em\relax Citeseer, 2007, pp. 817--824.

\bibitem{bubeck2012regret}
S.~Bubeck and N.~Cesa-Bianchi, ``Regret analysis of stochastic and
  nonstochastic multi-armed bandit problems,'' \emph{arXiv preprint
  arXiv:1204.5721}, 2012.

\bibitem{rusmevichientong2010linearly}
P.~Rusmevichientong and J.~N. Tsitsiklis, ``Linearly parameterized bandits,''
  \emph{Mathematics of Operations Research}, vol.~35, no.~2, pp. 395--411,
  2010.

\bibitem{dani2008stochastic}
V.~Dani, T.~P. Hayes, and S.~M. Kakade, ``Stochastic linear optimization under
  bandit feedback,'' 2008.

\bibitem{abe2003reinforcement}
N.~Abe, A.~W. Biermann, and P.~M. Long, ``Reinforcement learning with immediate
  rewards and linear hypotheses,'' \emph{Algorithmica}, vol.~37, no.~4, pp.
  263--293, 2003.

\bibitem{auer2002using}
P.~Auer, ``Using confidence bounds for exploitation-exploration trade-offs,''
  \emph{Journal of Machine Learning Research}, vol.~3, no. Nov, pp. 397--422,
  2002.

\bibitem{filippi2010parametric}
S.~Filippi, O.~Cappe, A.~Garivier, and C.~Szepesv{\'a}ri, ``Parametric bandits:
  The generalized linear case,'' in \emph{Advances in Neural Information
  Processing Systems}, 2010, pp. 586--594.

\bibitem{bubeck2011x}
S.~Bubeck, R.~Munos, G.~Stoltz, and C.~Szepesv{\'a}ri, ``X-armed bandits,''
  \emph{Journal of Machine Learning Research}, vol.~12, no. May, pp.
  1655--1695, 2011.

\bibitem{valko2013finite}
M.~Valko, N.~Korda, R.~Munos, I.~Flaounas, and N.~Cristianini, ``Finite-time
  analysis of kernelised contextual bandits,'' \emph{arXiv preprint
  arXiv:1309.6869}, 2013.

\bibitem{srinivas2009gaussian}
N.~Srinivas, A.~Krause, S.~M. Kakade, and M.~Seeger, ``Gaussian process
  optimization in the bandit setting: No regret and experimental design,''
  \emph{arXiv preprint arXiv:0912.3995}, 2009.

\bibitem{teerapittayanon2016branchynet}
S.~Teerapittayanon, B.~McDanel, and H.-T. Kung, ``Branchynet: Fast inference
  via early exiting from deep neural networks,'' in \emph{2016 23rd
  International Conference on Pattern Recognition (ICPR)}.\hskip 1em plus 0.5em
  minus 0.4em\relax IEEE, 2016, pp. 2464--2469.

\bibitem{konidaris2006autonomous}
G.~Konidaris and A.~Barto, ``Autonomous shaping: Knowledge transfer in
  reinforcement learning,'' in \emph{Proceedings of the 23rd international
  conference on Machine learning}, 2006, pp. 489--496.

\bibitem{torrey2010transfer}
L.~Torrey and J.~Shavlik, ``Transfer learning,'' in \emph{Handbook of research
  on machine learning applications and trends: algorithms, methods, and
  techniques}.\hskip 1em plus 0.5em minus 0.4em\relax IGI global, 2010, pp.
  242--264.

\bibitem{bhowmik2019estimagg}
A.~Bhowmik, M.~Chen, Z.~Xing, and S.~Rajan, ``Estimagg: A learning framework
  for groupwise aggregated data,'' in \emph{Proceedings of the 2019 SIAM
  International Conference on Data Mining}.\hskip 1em plus 0.5em minus
  0.4em\relax SIAM, 2019, pp. 477--485.

\bibitem{ILSVRC}
\BIBentryALTinterwordspacing
(2017) Large scale visual recognition challenge (ilsvrc). [Online]. Available:
  \url{http://www.image-net.org/challenges/LSVRC/}
\BIBentrySTDinterwordspacing

\bibitem{li2010contextual}
L.~Li, W.~Chu, J.~Langford, and R.~E. Schapire, ``A contextual-bandit approach
  to personalized news article recommendation,'' in \emph{Proceedings of the
  19th international conference on World wide web}, 2010, pp. 661--670.

\bibitem{hartford2017deep}
J.~Hartford, G.~Lewis, K.~Leyton-Brown, and M.~Taddy, ``Deep iv: A flexible
  approach for counterfactual prediction,'' in \emph{International Conference
  on Machine Learning}, 2017, pp. 1414--1423.

\end{thebibliography}
